\documentclass[journal]{IEEEtran}

\ifCLASSINFOpdf
\else
\fi

\usepackage{amsmath}
\usepackage{enumerate}
\usepackage{fancyhdr}
\usepackage{xcolor}
\usepackage{listings}
\usepackage{graphicx}
\usepackage{amssymb}
\usepackage{mathtools}

\usepackage{booktabs}
\usepackage{algorithm}
\usepackage{algorithmic}

\usepackage{enumitem}
\usepackage{tablefootnote}

\usepackage{subfigure}

\usepackage{comment}
\newtheorem{theorem}{Theorem}
\newtheorem{assumption}{Assumption}
\newtheorem{lemma}{Lemma}

\newtheorem{corollary}{Corollary}
\newtheorem{remark}{Remark}

\newtheorem{proof}{Proof}

\newcommand{\Prob}{\ensuremath{{\mathbb{P}}}}
\newcommand{\E}{\ensuremath{{\mathbb{E}}}}
\newcommand{\cb}{\mathcal{A}_{CB}}

\DeclareMathOperator*{\argmax}{argmax}

\begin{document}
%
\title{Model Selection for Generic Contextual Bandits}
%
%
%

\author{Avishek  Ghosh,
        Abishek Sankararaman
        and Kannan Ramchandran
\thanks{Avishek Ghosh is with the Systems and Control Engg. and the Centre for Machine Intelligence and Data Science (CMInDS) at IIT Bombay. }
\thanks{Abishek Sankararaman is with Amazon AWS AI, Palo Alto, USA}
\thanks{Kannan Ramchandran is with the EECS department, UC Berkeley}

\thanks{A part of this paper was presented at the International Conference on Artificial Intelligence and Statistics (AISTATS), 2021.}
\thanks{Contact avishek\_ghosh$@$iitb.ac.in for further questions. }
}

\markboth{IEEE Transactions on Information Theory}%
{Shell \MakeLowercase{\textit{et al.}}: Bare Demo of IEEEtran.cls for IEEE Journals}

\maketitle

\begin{abstract}
    We consider the problem of model selection for the general stochastic contextual bandits under the realizability assumption. We propose a successive refinement based algorithm called Adaptive Contextual Bandit ({\ttfamily ACB}), that works in phases and successively eliminates model classes that are too simple to fit the given instance. We prove that this algorithm is adaptive, i.e., the regret rate order-wise matches that of any  provable contextual bandit algorithm (ex. \cite{falcon}), that needs the knowledge of the true model class. The price of not knowing the correct model class turns out to be only an additive term contributing to the second order term in the regret bound. This cost  possess the intuitive property that it becomes smaller as the model class becomes easier to identify, and vice-versa.  We also show that a much simpler explore-then-commit (ETC) style algorithm also obtains similar regret bound, despite not knowing the true model class. However, the cost of model selection is higher in ETC as opposed to in {\ttfamily ACB}, as expected. Furthermore, for the special case of linear contextual bandits, we propose specialized algorithms that obtain sharper guarantees compared to the generic setup.
\end{abstract}


\section{Introduction}
\label{sec:intro}
The contextual Multi Armed Bandit (MAB) problem is a fundamental online learning setting capturing the exploration-exploitation trade-offs associated with sequential decision making (c.f. \cite{cesa2006prediction,chu2011contextual}). It consists of an agent, who at each time is shown a context by nature, and subsequently makes an irrevocable decision from a set of available decisions (arms) and collects a noisy reward depending on the arm chosen and the observed context.
The agent initially has no knowledge of the rewards of the various actions, and has to learn by repeated interaction over time,
the mapping from the set of context and arms to rewards.
The agent's goal is to minimize regret \textemdash the expected difference between the reward collected by an oracle that knows the expected rewards of all actions under all possible observed contexts and that of the agent. The recent books of \cite{lattimore2020bandit}, \cite{slivkins2019introduction} and the references therein provide comprehensive state-of-art on the general bandit problem.



We study the model selection question in general stochastic contextual bandits (c.f. \cite{agarwal2014taming}, \cite{agarwal2012contextual}, \cite{falcon}, \cite{foster2020beyond}). Practically, model selection in contextual bandits play a key role in applications such as personalized recommendation systems, which we sketch in the sequel in Section \ref{sec:motivating_example}. At a high-level, model selection is useful in deciding the function class (for example neural network architecture) to use to learn the mapping from contexts to rewards.
Smaller function class such as a logistic regression although are easier to train and tune hyper-parameters, may fit poorly to data (high statistical bias). On the other hand, very deep neural networks although in principle can achieve high statistical accuracy, incur overheads such as complex hyper-parameter tuning and challenges of explainability. The model selection problem formalizes this trade-off and defines an optimal choice (see Section \ref{sec:motivating_example}).

Formally, the contextual bandit setting is described as follows. At the beginning of each round $t \in [T]$, nature sequentially chooses a context $x_t \in \mathcal{X}$ and a reward function $r_t : \mathcal{A} \rightarrow [0,1]$ to an agent, who then subsequently takes an action $a_t \in \mathcal{A}$ from a finite set, and obtains a reward $r_t(a_t)$. In the stochastic setting (the focus of the present paper), the set of contexts and reward functions $\{x_t, r_t\}_{t = 1}^T$ are generated in an i.i.d. fashion from a distribution $D(x,r)$ which is apriori unknown to the agent. At each time $t$, conditional on the context $x_t$ and the action taken $a_t$, the observed reward $r_t(a_t)$ is independent of everything else, with mean $\mathbb{E}[r_t (a_t) \vert x_t, a_t] = f^*(x_t, a_t)$, where $f^*:\mathcal{X} \times \mathcal{A} \to [0,1]$, is an apriori unknown function. The agent is given a finite, nested family $(\mathcal{F}_m)_{m =1}^{M}$ of hypothesis classes\footnote{We use the term hypothesis class and model class interchangebly}, where $1\leq m_1 < m_2 \leq M$ implies $\mathcal{F}_{m_1}\subseteq \mathcal{F}_{m_2}$. Further, there exists an optimal class $d^{*} := \inf \{1\leq m \leq M : f^* \in \mathcal{F}_m\}$, i.e., $\mathcal{F}_{d^{*}}$ is the smallest hypothesis class containing the unknown reward function $f^*$. The agent is not aware of $d^{*}$ apriori and needs to estimate it. Model selection guarantees then refers to algorithms for the agent whose regret scales in the complexity of the \emph{smallest hypothesis class ($\mathcal{F}_{d^{*}}$ in the above notation) containing the true model}, even though the algorithm was not aware apriori.

In the case when the agent has the knowledge of $\mathcal{F}_{d^{*}}$ but does not know $f^*$, \cite{falcon} recently obtain computationally efficient algorithm {\ttfamily FALCON}, that achieves regret-rate scaling as $\sqrt{T}$. Using realizibility, i.e., $f^* \in \mathcal{F}_{d^*}$, it was shown in \cite{falcon}, that the stochastic contextual bandit can be reduced to an offline regression problem, which can be efficiently solved for many well known function classes beyond linear (eg. the set of all convex functions \cite{ghosh2019max}). The regret of FALCON was shown to scale proportional to the square root of the complexity of the function class $\mathcal{F}_{d^{*}}$ times $T$, the time horizon. In the case when $\mathcal{F}_{d^{*}}$ is a finite set, the complexity equals the logarithm of the cardinality, while if the class is infinite (either countable or uncountable), complexity is analogously defined (c.f. Section \ref{sec:general_infinite}).

The study in this paper is reliant on two assumptions: {\em (i) Realizability} (Assumption \ref{asm:second_moment}), \textemdash the true model belongs to at-least one of the many nested hypothesis classes, and {\em (ii) Separation} (Assumption \ref{asm:gap}) \textemdash the excess risk under any of the plausible model classes not containing the true model is strictly positive. Realizability, has been a standard assumption in stochastic contextual bandits (\cite{foster_model_selection}, \cite{foster2020beyond}, \cite{falcon}), and is used in our setup to define the optimal model class that needs to be selected. The separation assumption is needed to ensure that not selecting a relizable model class leads to regret scaling linear in time. The separation assumption is analogous to that used in standard multi-armed bandits \cite{lattimore2020bandit}, where the mean reward of the best arm is strictly larger than that of the second best arm.

\textbf{A negative result and the need of separability:} In \cite{marinov2021pareto}, the authors provide a negative answer to the open problem of \cite{foster2020open}, implying that it is not possible to obtain a regret which is order-wise identical to an oracle who knows the true model class $\mathcal{F}_{d^*}$. In particular, \cite{marinov2021pareto} shows that there always exists an instance where the regret in the smallest realizable class is (order-wise) larger than what is achievable with an oracle\footnote{The paper allows the hypothesis classes to be adversarily designed and hence in general requires a lot of exploration for model selection.}. This implies that if we aim to obtain oracle-optimal regret, we should exploit certain structures in the problem. In this paper, we achieve this by the \emph{separability} assumption, which comes naturally in statistical learning problems\footnote{The separability assumption restricts the amount the exploration needed, dependent on the gap or separation.}. This assumption should be thought as a first step towards obtaining (oracle) optimal regret.

In parallel independent work, \cite{krishnamurthy2021optimal} also study model selection problem, under the same assumptions of realizability and separation that we make. They propose {\ttfamily ModIGW} algorithm that is built on {\ttfamily FALCON} and shares similarity to our algorithm {\ttfamily ACB}; both algorithms run in epochs of doubling length, where at the beginning of each epoch, an appropriate model class is selected, and the rest of the epoch consists of playing {\ttfamily FALCON} on the selected model class.  In order to select the appropriate class, the nested structure of model classes along with the fact that the largest class $M$ is realizable by definition is used. The regret guarantees are similar for both {\ttfamily ACB} and {\ttfamily ModIGW}, with {\ttfamily ModIGW} having a better second order term, as they have a stronger assumption on the regression oracle. Remark \ref{rem:strong_asm} highlights that under the same assumption on the regression oracle, the second order term of {\ttfamily ACB} will match (order-wise) that of {\ttfamily ModIGW}. However, our proposed method, {\ttfamily ACB}  can be viewed as a meta-algorithm, that uses any state-of-art contextual bandit algorithm, $\cb$ as a black-box (see Algorithm~\ref{algo:gen}). In particular {\ttfamily ACB} works with any provable contextual bandit algorithm---a feature that {\ttfamily ModIGW} does not posses. Thus any improvement to the contextual bandit problem, automatically yields a model selection result through {\ttfamily ACB}. Moreover, our proof techniques are completely different to that of \cite{krishnamurthy2021optimal}.

Finally in Sections~\ref{sec:dimension_adaptation} and \ref{sec:dim_adap_finite}, we consider the specialized case where $f^*(.)$ assumes a linear form and thus parameterized by $\theta^* \in \mathbb{R}^d$. We note that in this setup the sparsity, $\|\theta^*\|_0$ naturally forms a nested hypothesis class, where $\mathcal{F}_i$ denotes the class of linear functions with sparsity $i$. So, we have $\mathcal{F}_1 \subseteq \mathcal{F}_2 \subseteq \ldots \subseteq \mathcal{F}_d$ and $M = d$. We propose and analyze a novel algorithm, namely Adaptive Linear Bandit-Dimension (\texttt{ALB-Dim}), which may be thought as a variant of the generic \texttt{ACB}. We show that the regret of  \texttt{ALB-Dim} scales linearly in the unknown cardinality of the support of $\theta^*$.  The regret of our algorithm matches that of an oracle who knows the  support of $\theta^*$ (\cite{sparse_bandit1},\cite{sparse_bandit2}), thereby achieving model selection guarantees.

We emphasize that the setting with dimension as a measure of complexity was also studied by \cite{sparse_bandit1}. 
However, our regret bounds are stronger (by a logarithm in $d$ factor). Furthermore, our algorithmic paradigm is more broadly applicable -- for eg. we can handle both the cases with finite as well as infinite arms, and obtain similar model selection regret guarantees that match the regret of an oracle that knows the true dimension.
Model selection with dimension as  complexity measure was also recently studied by \cite{foster_model_selection}, in which the classical contextual bandit (\cite{chu2011contextual}) with a finite number of arms was considered. We clarify here that although our results for the finite arm setting yields a better (optimal) regret scaling with respect to the time horizon $T$ and the support of $\theta^*$ (denoted by $d^*$), our guarantee depends on a problem dependent parameter and thus not uniform over all instances. In contrast, the results of \cite{foster_model_selection}, although sub-optimal in $d^*$ and $T$, is uniform over all problem instances. Closing this gap is an interesting future direction.

We emphasize here that our specialized algorithm, \texttt{ALB-Dim} does not require any (explicit) \emph{separability} assumption across hypothesis classes similar to the generic case. Also, our setup here can handle the case with finite as well as infinite number of arms/actions. Moreover, we show in Sections~\ref{sec:dimension_adaptation} and \ref{sec:dim_adap_finite} that the regret of \texttt{ALB-Dim} is independent (order-wise) of the number of actions, and hence for the finite action setup, it improves the regret of \texttt{ACB} by a factor of $\mathcal{O}(\sqrt{|\mathcal{A}|})$.

\subsection{Our Contributions}
\subsubsection{A Successive Refinement Algorithm for General Contextual Bandit}
We present {\ttfamily Adaptive Contextual Bandit (ACB)}, a meta algorithm that uses {\ttfamily FALCON} as a black box and show that its regret rate  matches (order-wise),  {\ttfamily FALCON}'s (\cite{falcon}), the state of art algorithm in contextual bandits which assumes knowledge of the true model class. 
{\ttfamily ACB} proceeds in epochs, with the first step in every epoch being a statistical test on the samples from the previous epoch to identify the smallest model class, followed by {\ttfamily FALCON} on this identified class in the epoch. We show that, with high probability, eventually, {\ttfamily ACB} identifies the true model class (Lemma \ref{lem:falcon_model}), and thus its regret rate matches that of {\ttfamily FALCON}. 
\paragraph{Cost of model selection:} The second order regret term in {\ttfamily ACB} scales as $O ( \frac{\log(T)}{\Delta^2})$, where $\Delta > 0$, is the gap (formally defined in Assumption \ref{asm:gap}) between the smallest class containing the true model and the highest model class not containing the true model. This term can be interpreted as the \emph{cost of model selection}. Furthermore, as this term is inversely proportional to the gap $\Delta$, we see that an `easier' instance ($\Delta$ being high), incurs lower cost of model selection than an instance with smaller $\Delta$. Furthermore, the model selection cost can be reduced to $\mathcal{O}(\frac{\log \log T}{\Delta^2} )$ if $T$ is known in advance.

\subsubsection{An Explore-then-commit (ETC) algorithm}
We propose and analyze an Explore-then-commit (ETC) algorithm that also achieves model selection, but requires knowledge of $T$ in advance has a larger second order regret compared to {\ttfamily ACB} . 
We show that a ETC algorithm also performs model selection, i.e., has a regret rate scaling as that of {\ttfamily FALCON} on the optimal model class. This is a conceptually simpler algorithm compared to {\ttfamily ACB}. In {\ttfamily ETC}, the model class is estimated once after a few rounds of forced exploration, and the rest of the time-horizon {\ttfamily FALCON} is played on the estimated model class. However, the cost of model selection in ETC is $O(\sqrt{T})$, which is larger than that of {\ttfamily ACB}. Nevertheless, asymptotically, a simple ETC algorithm suffices to obtain model selection.  

\subsubsection{ Improved Regret Guarantee with Linear Structure}
In the special setup of stochastic linear bandits, where the reward is a linear map of the context, we propose and analyze an adaptive algorithm, namely Adaptive Linear Bandits-Dimension (\texttt{ALB-Dim}). First we observe that in this special case, we do not require any \emph{separability} assumption. Moreover, the setup of linear bandits can include both finite as well as infinite number of actions. We show that the regret of \texttt{ALB-Dim} is independent of the number of actions (arms), which is an improvement over the regret of \texttt{ACB}. In particular, for the finite arm setup, \texttt{ALB-Dim} improves the regret of \texttt{ACB} by a factor of $\mathcal{O}(\sqrt{|\mathcal{A}|})$.

\subsection{Motivating example}
\label{sec:motivating_example}

Model selection in contextual bandits plays a key role in applications such as personalized recommendation systems, which we sketch. Consider a system (such as news recommendation) that on each day, recommends one out of $K$ possible outlets to a user. On each day, an event is realized in nature, which can be modeled as the context vector on that day. The true model function $f^{*}$ encodes the user's preference; for example the user prefers one outlet for sports oriented articles, while another for international events. This apriori unknown to the system and needs to learn this through repeated interactions. The multiple nested hypothesis classes correspond to a variety of possible neural network architectures to learn the mapping from contexts (event of the day) to rewards (which can be engagement with the recommended item). In practice, these nested hypothesis classes range from simple logistic regression to multi-layer perceptrons \cite{cheng2016wide}. Complex network architectures although has the potential for increased accuracy, incurs undesirable overheads such as requiring larger offline training to deliver accuracy gains \cite{cheng2016wide}, computational complexity in hyper-parameter tuning \cite{caselles2018word2vec} and challenges of explainability in predictions \cite{mcinerney2018explore,balog2019transparent}. Model selection provides a framework to trade-off between accuracy and the overheads.
\section{Related work}
\label{sec:related_work}

Model selection for MAB have received increased attention in recent times owing to its applicability in a variety of large-scale settings such as recommendation systems and personalization. The special case of linear contextual bandits was studied in \cite{osom}, \cite{ghosh2021problem} and \cite{foster_model_selection}, where both instance dependent and instance independent algorithms achieving model selection were given. In \cite{osom}, \cite{ghosh2021problem}, the standard OFUL algorithm of \cite{oful} is taken as a baseline and model selection procedures are proposed on top of that. In this linear bandit framework, similar to the present paper, \cite{foster_model_selection} and \cite{ghosh2021problem} considered the family of nested hypothesis classes, with each class positing the sparsity of the unknown linear bandit parameter. In this setup, \cite{foster_model_selection} proposed {\ttfamily ModCB} which uses the Exp4-IX algorithm of \cite{neu2015explore} as a base algorithm and achieves regret rate uniformly for all instances, a rate that is sub-optimal compared to the oracle that knows the true sparsity. In contrast, both our paper and \cite{ghosh2021problem} propose an algorithm that achieves regret rate matching that of the oracle that knows the true sparsity. The cost of model selection contributes only a constant that depends on the instance but independent of the time horizon. However, unlike {\ttfamily ModCB}, our regret guarantees are problem dependent  and do not hold uniformly for all instances. 
A parallel line of work on linear bandits has focused on simple LASSO type algorithms under strong stochastic assumptions on the distribution of the contexts that achieve model selection guarantees \cite{sparse_bandit2, bastani2021mostly, oh2020sparsity, ariu2020thresholded, li2021simple}.

A black-box model selection framework for MABs called {\ttfamily Corral} was proposed in \cite{agarwal2017corralling}, where the optimal algorithm for each hypothesis class is treated as an expert and the task of the forecaster is to have low regret with respect to the best expert (best model class). The generality of this framework has rendered it fruitful in a variety of different settings; for example \cite{agarwal2017corralling, arora2021corralling} considered unstructured MABs, which was then extended to both linear and contextual bandits and linear reinforcement learning in a series of works \cite{pacchiano2020regret,aldo-corral} and lately to even reinforcement learning \cite{lee2021online}. However, the price for this versatility is that the regret rates the cost of model selection is multiplicative rather than additive. In particular, for the special case of linear bandits and linear reinforcement learning, the regret scales as $\sqrt{T}$ in time with an additional multiplicative factor of $\sqrt{M}$, while the regret scaling with time is strictly larger than $\sqrt{T}$ in the general contextual bandit. Since this approach treats all the hypothesis classes as bandit arms, and work in a (restricted) partial information setting, they tend to explore a lot, yielding worse regret. On the other hand, we consider all $M$ classes at once (full information setting) and do inference, and hence explore less and obtain lower regret. Recently, the above idea of regret balancing is extended to black box optimization in the context of non-stationary Reinforcement Learning (\cite{wei2021non}) and robust Reinforcement Learning (\cite{wei2022model}).

\begin{table*}
\begin{center}
\begin{tabular}{c | c | c |c| c} 
  & Regret Bound  & Function Class & Arms & Base Algorithm  \\ 
 \hline

  \cite{osom} & $\widetilde{\mathcal{O}}(\sqrt{T})$ & $M=2$, Linear & Finite & \texttt{OFUL} \\

\cite{foster_model_selection} & $\widetilde{\mathcal{O}}(T^{2/3}(Kd_{m^*})^{1/3})$ & Linear & Finite & \texttt{Exp4-IX} \\

 \cite{pacchiano2020model} & $\widetilde{\mathcal{O}}(\sqrt{MT} + d_{m^*} \sqrt{m^*T})$ & Generic & Infinite & \texttt{CORRAL}  \\ 
 
 \cite{krishnamurthy2021optimal} &$\widetilde{\mathcal{O}}(d_{m^*}^2 + \sqrt{Kd_{m^*}T})$ & Generic & Finite & \texttt{FALCON}\\

 This paper & $\widetilde{\mathcal{O}}(d_{M} + \sqrt{Kd_{m^*}T})$  & Generic & Finite & Generic ($\cb$) \\ 
\hline
\end{tabular}
\end{center}
\vspace{2mm}
\caption{Table comparing related work on Model Selection for Contextual Bandits. Here $d_{m^*}$ corresponds to the complexity measure (ex. dimension for linear bandits, log cardinality for finite function classes) of the smallest hypothesis class containing the true regressor $f^*$. Also, $d_M$ referes to the complexity of the largest hypothesis class $\mathcal{F}_M$. We see that our results are competitive with respect to the existing works, and can handle any generic contextual bandit algorithm, $\mathcal{A}_{CB} $ as opposed to \cite{krishnamurthy2021optimal} which uses \texttt{FALCON}.}
\end{table*}

Furthermore, \cite{vidya} study the problem of model selection in RL with function approximation. Similar to the \emph{active-arm elimination} technique employed in standard multi-armed bandit (MAB) problems \cite{eliminate}, the authors eliminate the model classes that are dubbed misspecified, and obtain a regret of $\mathcal{O}(T^{2/3})$. On the other hand, our framework is quite different in the sense that we consider model selection for generic contextual bandits. Moreover, our regret scales as $\mathcal{O}(\sqrt{T})$.

Adaptive algorithms for linear bandits have also been studied in different contexts from ours. The papers of \cite{locatelli,krishnamurthy2} consider problems where the arms have an unknown structure, and propose algorithms adapting to this structure to yield low regret. The paper \cite{easy_data2} proposes an algorithm in the adversarial bandit setup that adapt to an unknown structure in the adversary's loss sequence, to obtain low regret.
The paper of \cite{temporal} consider adaptive algorithms, when the distribution changes over time. 
In the context of online learning with full feedback, there have been several works addressing model selection \cite{online_mod_sel1,online_mod_sel2,online_mod_sel3,online_mod_sel4}.
In the context of statistical learning, model selection has a long line of work (for eg. \cite{vapnik_book}, \cite{massart}, \cite{lugosi_adaptive}, \cite{arlot2011margin}, \cite{cherkassky2002model} \cite{devroye_book}). 
However, the bandit feedback in our setups is much more challenging and a straightforward adaptation of algorithms developed for either statistical learning or full information to the setting with bandit feedback is not feasible.

\section{Problem formulation}
\label{sec:formulation}
\paragraph{Setup:} Let $\mathcal{A}$ be the set of $K$ actions, and let $\mathcal{X} \subseteq \mathbb{R}^d$ be the set of $d$ dimensional contexts. At time $t$, nature picks $(x_t,r_t)$ in an i.i.d fashion from an unknown distribution $D(x,r)$ (see \cite{agarwal2012contextual}), where $x_t \in \mathcal{X}$ and a context dependent $r_t: \mathcal{A}\rightarrow [0,1]$. All expectation operators in this section are with respect to this i.i.d. sequence $(x,r)$. Upon observing the context, the agent takes action $a_t \in \mathcal{A}$, and obtains the reward of $r_t(a_t)$. Note that, the reward $r_t(a_t,x_t)$ depends on the context $x_t$ and the action $a_t$. Furthermore, it is standard (\cite{foster_model_selection,falcon}) to have a realizibility assumption on the conditional expectation of the reward, i.e., there exists a predictor $f^* \in \mathcal{F}$, such that $\mathbb{E}[r_t(a,x)|x_t = x,a] = f^*(x,a)$, for all $x$ and $a$. We suppress the dependence of the reward on the context $x_t$ and denote the reward at time $t$ from action $a \in \mathcal{A}$ as $r_t(a)$.

In the contextual bandit literature (\cite{agarwal2012contextual,falcon}) it is generally assumed that the true regression function $f^*$ is unknown, but  the function class $\mathcal{F}$ where it belongs, is known to the learner. The price of not knowing $f^*$ is characterized by regret, which we define now. To set up notation, for any $f \in \mathcal{F}$, we define a policy induced by the function $f$, $\pi_{f}:\mathcal{X} \to \mathcal{A}$ as $\pi_{f}(x) = \mathrm{argmax}_{a \in \mathcal{A}} f(x,a)$ \footnote{Ties are broken arbitrarily, for example the lexicographic ordering of $\mathcal{A}$}, for all $x \in \mathcal{X}$. We define the regret over $T$ rounds defined as 
\begin{align*}
    R(T) = \sum_{t=1}^T [r_t(\pi_{f^*}(x_t)) - r_t(a_t)]
\end{align*}
Throughout this paper, we obtain high probability bounds on $R(T)$    

\section{Model selection for generic contextual bandits}
\label{sec:general}
In this section, we focus on the main contribution of the paper---a provable model selection guarantee for the (generic) stochastic contextual bandit problem. In contrast to the standard setting, in the  model selection framework, we do not know $\mathcal{F}$. Instead, we are given  a nested class of $M$ function classes, $\mathcal{F}_1 \subset \mathcal{F}_2 \subset \ldots \subset \mathcal{F}_M$. Let the smallest function class where the true regressor, $f^*$ lies be denoted by $\mathcal{F}_{d^*}$, where $d^* \in [M]$. 

From the above discussion, since $f^* \in \mathcal{F}_{d^*}$,  the regret of an \emph{adaptive} contextual bandit algorithm should depend on the function class $\mathcal{F}_{d^*}$. However, we do not know $d^*$, and our goal is to propose adaptive algorithms such that the regret depends on the \emph{actual} problem complexity $\mathcal{F}_{d^*}$. First, let us write the realizability assumption with the nested function classes.

\begin{assumption} [Realizability]
\label{asm:second_moment}
There exists $1 \leq d^{*} \leq M$, and a predictor $f^* \in \mathcal{F}_{d^*}$, such that $\mathbb{E}[r_t(a)|x_t = x] = f^*(x,a)$, for all $x \in \mathcal{X}$ and $a \in \mathcal{A}$. 
\end{assumption}

Furthermore, in order to identify the correct model class within the given $M$ hypothesis classes, we also require the following separability condition. The motivation of the separability comes from the following negative result.

Very recently, \cite{marinov2021pareto}  provides a negative answer to the open problem of \cite{foster2020open}, showing  that it is not possible to obtain a regret which is order-wise identical to an oracle who knows the true model class $\mathcal{F}_{d^*}$. Specifically, \cite{marinov2021pareto} shows that there always exists an instance where the regret in the smallest realizable class is (order-wise) larger than of an oracle.

\begin{assumption}[Separability]
\label{asm:gap}
There exists a $\Delta > 0$, such that,
\begin{align*}
   \inf_{f \in \mathcal{F}_{d^* -1}} \mathbb{E}_x \left[ \inf_{a \in \mathcal{A}}  [{f}(x,a) - f^*(x,a)]^2 \right] \geq \Delta.
\end{align*}
The parameter $\Delta > 0$ is the minimum separation across the function classes. The expectation above is with respect to the randomness in contexts.
\end{assumption}



Note that the identical separability condition is also witnessed in \cite{krishnamurthy2021optimal}\footnote{This is equivalent to \\ $\inf_{f \in \mathcal{F}_{d^* -1}} \mathbb{E}_x \left[ \inf_{q:\mathcal{X}\rightarrow \Delta(\mathcal{A})} \inf_{a \sim q(x)}  [{f}(x,a) - f^*(x,a)]^2 \right] \geq \Delta$}. The above condition implies that there is a (non-zero) gap, between the regressor functions belonging to the realizable classes and non-realizable classes. Since, we have nested structure, $\mathcal{F}_1 \subset \mathcal{F}_2 \subset \ldots \subset \mathcal{F}_M$, condition on the biggest non-realizable class, $\mathcal{F}_{d^*-1}$ is sufficient. We emphasize that separability condition is quite standard in statistics, specially in the area of clustering (\cite{lu2016statistical}), analysis of Expectation Maximization (EM) algorithm (\cite{em_1,em_2}, understanding the behavior of Alternating Minimization (AM) algorithms (\cite{mixture-many,ghosh2019max}).

Having said that, we believe a weaker separability assumption that requires $f \in \mathcal{F}_{d^* -1} $ and $f^*$ to be separated \emph{near} the optimal action $\pi_{f^*}$ only (local separability) should be sufficient---which is often the case for (offline) statistical problems. However, with finite ($K$) number of actions, it is not immediately clear how to weaken this, and model selection without (or with weak) separability is kept as an interesting future work. We also emphasize that although we require the gap assumption for theoretical analysis, our algorithm (described next) does not require any knowledge of $\Delta$, and adapts to the gap of the problem.

\subsection{Warm Up: A simple Explore-Then-Commit (ETC) algorithm for model selection}
\label{sec:etc}


In this section, we provide a simple model selection algorithm based on Explore-Then-Commit (ETC) novel model selection algorithm that use successive refinements over epochs. We use a simple Explore-Then-Commit (ETC) algorithm for selecting the correct function class, and then commit to it during the exploitation phase. After a round of exploration, we do a (one-time) threshold based testing to estimate the function class, and after that, exploit the estimated function class for the rest of the iterations. Here, we consider any (generic) contextual bandit algorithm $\mathcal{A}_{CB}$ along with the function class $\mathcal{F}$ containing the true regressor $f^*$. The details are provided in Algorithm~\ref{algo:main_algo}.

As an example of $\cb$, we use a provable contextual bandit algorithm, namely FALCON (stands for FAst Least-squares-regression-oracle CONtextual bandits) of \cite{falcon}, the details are provided in Algorithm~\ref{algo:falcon}. 

Note that in this section, for simplicity, we continue to consider consider the setup where the function classes $\mathcal{F}_1,\ldots,\mathcal{F}_M$ are finite. However, in Section~\ref{sec:general_infinite}, we remove this, and work in infinite function classes.

We show that this simple strategy finds the optimal function class $\mathcal{F}_{d^*}$ with high probability. We now explain the exploration and exploitation phases of this algorithm.

\begin{algorithm}[t!]
  \caption{\texttt{ETC} for model selection for contextual bandits}
  \begin{algorithmic}[1]
 \STATE  \textbf{Input:} Function classes $\mathcal{F}_1 \subset \mathcal{F}_2 \subset \ldots \subset \mathcal{F}_M$, time horizon $T$, confidence parameter $\delta$
 \STATE \textbf{Explore:}
 \FOR {$t =1,2,\ldots, \lceil \sqrt{T} \rceil $}
 \STATE Observe context reward pair $(x_t,r_t)$
 \STATE Select action $a_t$ uniformly at random from $\mathcal{A}$, independent of $x_t$
 \STATE Observe reward $r_t(a_t)$
 \ENDFOR
 \STATE Compute regression estimator $\hat{f}_j = \mathrm{argmin}_{f \in \mathcal{F}_j} \frac{1}{\sqrt{T}} \sum_{t=1}^{\lceil\sqrt{T}\rceil}[f(x_t,a_t) - r_t(a_t)]^2$ (via offline regression oracle) for all $j \in [M]$
 \STATE \textbf{Model Selection test:}
 \STATE Obtain another set of $\lceil\sqrt{T}\rceil$ fresh samples of $(x_t,r_t,a_t)$ via pure exploration (similar to line 4-6 )
 \STATE Construct the test statistic $S_j = \frac{1}{\lceil\sqrt{T}\rceil}\sum_{t=1}^{\lceil\sqrt{T}\rceil} (\hat{f}_j (x_t,a_t) - r_t(a_t))^2 $ for all $j \in [M]$
 \STATE \textbf{Thresholding}: Find minimum index $\ell \in [M]$ such that $S_j \leq S_M + \frac{\sqrt{\log T}}{T^{1/4}}$ and obtain  $\hat{f}_\ell \in \mathcal{F}_{\ell}$
 \STATE \textbf{Commit:}
 \FOR {$t = 2\lceil\sqrt{T}\rceil+1,\ldots, T $} 
 \STATE Observe context $x_t \in \mathcal{X}$ and reward function $r_t$
 \STATE Run $\cb(\mathcal{F}_{\ell})$
 \STATE Obtain $a_t$ and observe reward $r_t(a_t)$.
 \ENDFOR
  \end{algorithmic}
  \label{algo:main_algo}
\end{algorithm}



For the first $2\sqrt{T}$ time epochs, we do the exploration (i.e., sample  randomly). Precisely, the context-reward pair $(x_t,r_t)$ is being sampled by nature in an i.i.d fashion, and the action the agent takes is chosen uniformly at random from the action set $\mathcal{A}$. In particular, the action is chosen independent of the context $x_t$. Hence, this is a pure exploration strategy. 
 
Based on the samples of the first  $\sqrt{T}$ rounds, we estimate the regression function $\{\hat{f}_j\}_{j=1}^M$ for all the (hypothesis) function classes $\mathcal{F}_1,\ldots,\mathcal{F}_M$ via offline regression oracle (see \cite{falcon} for details) and obtain $\hat{f}_j = \mathrm{argmin}_{f \in \mathcal{F}_j}(\sum_{t=1}^{\sqrt{T}} f(x_t,a_t) -r_t(a_t))^2$ for all $j \in [M]$.

 To remove dependence issues, we use the remaining $\sqrt{T}$ samples obtained form the sampling phase. Here we actually compute the following test statistic for all hypothesis classes, namely
$$S_j = \frac{1}{\sqrt{T}}\sum_{t=1}^{\sqrt{T}} (\hat{f}_j (x_t,a_t) - r_t(a_t))^2 $$
for all $j \in [M]$. We then perform a thresholding on $\{S_j\}_{j=1}^M$. We pick the smallest index $j$ such that $S_j \leq S_M + \frac{\sqrt{\log T}}{T^{1/4}}$. We then commit to this function class for the rest $T-2\sqrt{T}$ time steps. Hence, in Algorithm~\ref{algo:main_algo}, we perform one step thresholding and commit to it. We show that simple scheme obtains the correct model with high probability.
\subsection{Regret Guarantee of ETC}
\begin{lemma}[Model Selection for \texttt{ETC}]
 \label{lem:etc}
 Suppose the time horizon satisfies
 \begin{align*}
     T \gtrsim (\log T) \,\, \max \left( \log \left( \sqrt{T}{|\mathcal{F}_M|} \right), \Delta^{-4}, \log(1/\delta) \right).
 \end{align*} 
 Then with probability at least $1-4M\delta$, line $11$ in Algorithm~\ref{algo:main_algo} identifies the correct model class $\mathcal{F}_{d^*}$.
 \end{lemma}


We now analyze the regret performance of Algorithm~\ref{algo:main_algo}. The regret $R(T)$ is comprised of $2$ stages; (a) exploration and (b) commit (exploitation). We have the following result.

\begin{theorem}
\label{thm:etc}
Suppose Assumptions~\ref{asm:second_moment} and \ref{asm:gap} hold. Then with probability at least $1-4M\delta$, running Algorithm~\ref{algo:main_algo} for $T$ iterations yield
\begin{align*}
    R(T) \leq  C \,\, \sqrt{T} + R_{\mathcal{A}_{CB}(\mathcal{F}_{d^*})} (T-2\sqrt{T}),
\end{align*}
where $R_{\mathcal{A}_{CB}(\mathcal{F}_{d^*})} (T-2\sqrt{T})$ is the regret of the $\cb$ with function class $\mathcal{F}_{d^*}$. In particular, if $\cb =$ FALCON, with probability at least $1-4M\delta-\delta$, we obtain
\begin{align*}
      R(T) &\leq C \sqrt{T} + \mathcal{O} \left(\sqrt{KT \log(|\mathcal{F}_{d^*}|T/\delta)}\right).
\end{align*}
\end{theorem}
\begin{remark}[Cost of Model Selection]
As seen in Theorem~\ref{thm:etc}, the cost of model selection is $\mathcal{O}(\sqrt{T})$. In the next section, we propose a successive refinement based algorithm to cut down this cost to $\mathcal{O}(\log T)$.
\end{remark}

\begin{remark} [Matches Oracle]
Let us consider the special case when $\cb$= FALCON. In the regret expression, the first term scales with $\mathcal{O}(\sqrt{T})$. The second expression in the regret is $\Tilde{\mathcal{O}}(\sqrt{KT \log(|\mathcal{F}_{d^*}|T/\delta)}$, with high probability). So, we observe that (order-wise) the cost of model selection is no-worse than the regret of FALCON even with the knowledge of the smallest function class containing $f^*$, i.e, $\mathcal{F}_{d^*}$.
\end{remark}

\subsection{Beyond ETC: Algorithm---Adaptive Contextual Bandits  (\texttt{ACB})}
\label{sec:adaptive_model_sel}
In the previous section, we saw a simple ETC type algorithm for model selection. In this section, we propose and analyze a novel model selection algorithm that use successive refinements over epochs to cut down the cost of model selection.  Similar to the previous section, we consider any contextual bandit algorithm $\mathcal{A}_{CB}$ along with the function class $\mathcal{F}$ containing the true regressor $f^*$. We take $\cb(\mathcal{F})$ as a baseline, and add a model selection phase at the beginning of each epoch. In other words, over multiple epochs, we successively refine our estimates of the \emph{proper} model class where the true regressor function $f^*$ lies. The details are provided in Algorithm~\ref{algo:gen}. Note that \texttt{ACB} does not require any knowledge of the separation $\Delta$. 

As an example of $\cb$, we use a provable contextual bandit algorithm, namely FALCON (stands for FAst Least-squares-regression-oracle CONtextual bandits) of \cite{falcon}, the details are provided in Algorithm~\ref{algo:falcon}. 

Note that in this section, for simplicity, we continue to consider consider the setup where the function classes $\mathcal{F}_1,\ldots,\mathcal{F}_M$ are finite. However, in Section~\ref{sec:general_infinite}, we remove this, and work in infinite function classes.

\paragraph{The Base Algorithm:}
We work with a generic contextual bandit algorithm, $\cb$, which, upon observing context $x_t$, outputs an action $a_t$ for the agent along with the reward $r_t(a_t)$. As a special case, we take the example of a contextual bandit algorithm, FALCON (see Algorithm~\ref{algo:falcon}),  which is recently proposed and analyzed in \cite{falcon}. In particular, FALCON gives provable guarantees for contextual bandits beyond linear structure. FALCON is an epoch based algorithm, and depends only on an \emph{offline regression oracle}, which outputs an estimate $\hat{f}$ of the regression function $f^*$ at the beginning of each epoch. FALCON then uses a randomization scheme, that depends on the inverse gap with respect to the estimate of the best action.  Suppose that the true regressor $f^* \in \mathcal{F}$, and the realizibility condition (Assumption~\ref{asm:second_moment}) holds. With a proper choice of learning rate, with probability $1-\delta$, FALCON yields a regret of $
    R(T) \leq \mathcal{O}(\sqrt{K T \log(|\mathcal{F}|T/\delta)})$. Although the above result makes sense only for the finite $\mathcal{F}$, an extension to the infinite $\mathcal{F}$ is possible and was addressed in the same paper (see \cite{falcon}).


\paragraph{Our Approach} We use successive refinement based model selection strategy along with the base algorithm $\cb$. The details of our algorithm, namely Adaptive Contextual Bandits (\texttt{ACB}) are given in Algorithm~\ref{algo:gen}. We break the time horizon into several epochs with doubling epoch length. Let $\tau_0,\tau_1,\ldots$ be epoch instances, with $\tau_0 = 0$, and $\tau_m = 2^m$. Before the beginning of the $m$-th epoch, using all the data of the $m-1$-th epoch, we add a model selection module, as shown in Algorithm~\ref{algo:gen} (lines 4-8).

\begin{algorithm}[t!]
  \caption{Adaptive Cotextual Bandits (\texttt{ACB})}
  \begin{algorithmic}[1]
 \STATE  \textbf{Input:} epochs $0=\tau_0<\tau_1<\tau_2<\ldots$, confidence parameter $\delta$, Function classes $\mathcal{F}_1 \subset \mathcal{F}_2 \subset \ldots \subset \mathcal{F}_M$
 \FOR {epoch $m =1,2,\ldots, $}
 \STATE $\delta_m = \delta/2^m$
 \FOR {function classes $j = 1,2,\ldots,M$}
 \STATE Compute $\hat{f}_j^m = \mathrm{argmin}_{f \in \mathcal{F}_j}\sum_{t = \tau_{m-2}+1}^{\tau_{m-1}/2} (f(x_t,a_t) - r_t(a_t))^2 $ via offline regression oracle
 \STATE Construct  $S_j^m = \frac{1}{2^{m-2}}\sum_{t=\tau_{m-1}/2 +1}^{\tau_{m-1}} (\hat{f}_j^m (x_t,a_t) - r_t(a_t))^2 $
 \ENDFOR
 \STATE \textbf{Model Selection:} Find the minimum index $j \in [M]$ such that $S_j^m \leq S_M^m + \frac{\sqrt{m}}{2^{m/2}}$. Let this index be $\ell$ and the class be  $\mathcal{F}_{\ell}^m$
 \FOR {round $t=\tau_{m-1}+1,\ldots,\tau_m$}
\STATE Observe context $x_t \in \mathcal{X}$ and reward function $r_t$
 \STATE Run $\cb(\mathcal{F}_{\ell}^m)$
 \STATE Obtain $a_t$ and observe reward $r_t(a_t)$.
 \ENDFOR
 \ENDFOR
  \end{algorithmic}
  \label{algo:gen}
\end{algorithm}

\begin{algorithm}[t!]
  \caption{Special Case: $\cb(\mathcal{F}_{\ell}^m)=$ FALCON$(\mathcal{F}_{\ell}^m)$ at time $t$}
  \begin{algorithmic}[1]
 \STATE  \textbf{Input:} epochs $0=\tau_0<\tau_1<\tau_2<\ldots$, epoch index $m$, Hypothesis class: $\mathcal{F}_{\ell}^m$, confidence parameter $\delta_m$
 \STATE Set learning rate $\rho_m = \frac{1}{30} \sqrt{K(\tau_{m-1}-\tau_{m-2})/\log(|\mathcal{F}_\ell^m| (\tau_{m-1}-\tau_{m-2})m/\delta_m)}$
 \STATE Observe context $x_t \in \mathcal{X}$
 \STATE Compute $\hat{f}^m_\ell(a)$ for all action $a \in \mathcal{A}$, set $\hat{a}_t=\mathrm{argmax}_{a\in \mathcal{A}}\hat{f}^m_\ell(a)$
 \STATE Define $p_t(a) = \frac{1}{K+\rho_m(\hat{f}^m_\ell(x_t,\hat{a}_t) - \hat{f}^m_\ell(x_t,a)}\,\, \forall a \neq \hat{a}_t$, \hspace{2mm}  $p_t(\hat{a}_t) = 1-\sum_{a\neq \hat{a}_t}p_t(a)$.
 \STATE Sample $a_t \sim p_t(.)$ and observe reward $r_t(a_t)$.
  \end{algorithmic}
  \label{algo:falcon}
\end{algorithm}


Note that, in \texttt{ACB},  we feed the samples of the $m-1$-th epoch to the offline regression oracle. Moreover, we split the samples in 2 equal halves. We use the first half to compute the regression estimate
\vspace{-2mm}
\begin{align*}
    \hat{f}_j^m = \mathrm{argmin}_{f \in \mathcal{F}_j}\sum_{t = \tau_{m-2}+1}^{\tau_{m-1}/2} (f(x_t,a_t) - r_t(a_t))^2
\end{align*}
via offline regression oracle. \texttt{ACB} then use the rest of the samples to construct the test statistics given by,
\vspace{-2mm}
\begin{align*}
S_j^m = \frac{1}{2^{m-2}}\sum_{t=\tau_{m-1}/2 +1}^{\tau_{m-1}} (\hat{f}_j^m (x_t,a_t) - r_t(a_t))^2  
\vspace{-2mm}
\end{align*}
for all $j \in [M]$. We do not use the same set of samples to remove any dependence issues with $\hat{f}_j^m$ and the samples $\{x_t,a_t,r_t(a_t)\}_{t = \tau_{m-1}/2 +1}^{\tau_{m-1}}$.

\texttt{ACB} then compares the test statistics $\{S_j^{m}\}_{m=1}^M$ in Line $8$ of Algorithm \ref{algo:gen} to pick the model class. Intuitively, we expect $S_j^m$ to be small for all hypothesis classes that contain $f^*_{d^*}$. Otherwise, thanks to the separation condition in Assumption \ref{asm:gap}, we expect $S_j^m$ to be large. Realizability, i.e., Assumption \ref{asm:second_moment} ensures that $\mathcal{F}_M$, the largest hypothesis class by definition contains the true model $f^{*}$. Thus $S_M^{m}$ serves as an estimate of how small the excess risk of any realizable class must be. We set the threshold to be a small addition to $S_M^m$. The additional term of $\sqrt{\frac{m}{2^m}}$ in Line $8$ of Algorithm \ref{algo:gen} is chosen so that it is not too small, but nevertheless goes to $0$, as $m\to \infty$. In particular, we choose the threshold in {\ttfamily ACB} such that it is large enough to ensure all realizable classes have excess risk smaller than this threshold, but also not so large that it exceeds the excess risk of the non-realizable classes.

Let $\mathcal{F}_\ell^m$ be function class selected by this procedure in epoch $m$. \texttt{ACB} now uses the base algorithm, $\cb(\mathcal{F}_\ell^m)$ to obtain an action $a_t$ and corresponding reward $r_t(a_t)$. For instance, in the case of FALCON (as seen in Algorithm~\ref{algo:falcon}), the learner uses \emph{inverse gap} randomization with properly chosen learning rate (see \cite{falcon,foster20a,sen2021top}) to select the action $a_t$. In particular, with $\Hat{f}^m_\ell$ as the regressor function, let $\hat{a}_t=\mathrm{argmax}_{a\in \mathcal{A}}\hat{f}^m_\ell(a)$ be the greedy action. The \emph{inverse gap} randomization $p_t(.)$ is defined in the following way:
\begin{align*}
    & p_t(a) = \frac{1}{K+\rho_m(\hat{f}^m_\ell(x_t,\hat{a}_t) - \hat{f}^m_\ell(x_t,a)}\,\, \forall a \neq \hat{a}_t, \\ & p_t(\hat{a}_t) = 1-\sum_{a\neq \hat{a}_t}p_t(a),
\end{align*}
where $K$ is the number of arms (actions) and $\rho_m$ is the learning rate. Finally, we sample action $a_t \sim p_t(.)$ and henceforth observe reward $r_t(a_t)$.

\subsection{Analysis of \texttt{ACB}}
We now analyze the performance of the model selection procedure of Algorithm~\ref{algo:gen}. We have the doubling epochs, i.e., $\tau_m = 2^m$. Without loss of generality, we simply assume $\tau_1 =2$. Also, assume that we are at the beginning of epoch $m$, and hence we have the samples from epoch $m-1$. So, we have total of $2^{m-1}$ samples, out of which, we use $2^{m-2}$ to construct the regression functions and the rest $2^{m-2}$ to obtain the testing function $S_j^m$. Furthermore, we want the model selection procedure to succeed with probability at least $1-\delta/2^m$, since the we want a guarantee that holds for all $m$, and a simple application of the union bound yields that. We first show that \texttt{ACB} identifies the correct function class with high probability after a few epochs. We have the following Lemma.

\begin{lemma} [Model Selection of \texttt{ACB}]
\label{lem:falcon_model}
Suppose Assumptions~\ref{asm:second_moment} and \ref{asm:gap} holds and we run Algorithm~\ref{algo:gen}. Then, in all phases $m$ such that
\vspace{-2mm}
{\color{black}\begin{align*}
     2^m &\gtrsim  \max\{\frac{\log T}{\Delta^2}, \log (|\mathcal{F}_{M}|), \log (1/\delta)\} 
\end{align*}}
Algorithm~\ref{algo:gen} identifies the correct model class $\mathcal{F}_{d^*}$ in Line $8$, with probability exceeding $1-2M\delta$.
\end{lemma}
\textit{Proof sketch.} In order select the correct function class, we first obtain upper bounds on the test statistics $S^{(m)}_j$ for model classes that includes the true regressor $f^*_{d^*}$. We accomplish this by first carefully bounding the expectation of $S^{(m)}_j$ and then using concentration. We then obtain a lower bound on $S^{(m)}_j$ for model classes not containing $f^*_{d^*}$ via leveraging Assumption~\ref{asm:gap} (separability) along with Assumption~\ref{asm:second_moment}. Combining the above two bounds yields the desired result.

\paragraph{Regret Guarantee} With the above lemma, we obtain the following regret bound for Algorithm~\ref{algo:gen}. 
\begin{theorem}
\label{thm:falcon}
Suppose the conditions of Lemma~\ref{lem:falcon_model} hold. Then with probability at least $1-2M\delta$, running Algorithm~\ref{algo:gen} for $T$ iterations yield
\begin{align*}
      R(T) &\leq C \max\{\frac{\log T}{\Delta^2}, \log (|\mathcal{F}_{M}|), \log (1/\delta)\}  \\
      &+ R_{\cb(\mathcal{F}_{d^*})}(T)
      \end{align*}
where $R_{\cb(\mathcal{F}_{d^*})}(T)$ is the regret of $\cb$ with hypothesis class $\mathcal{F}_{d^*}$. In particular, if $\cb =$ FALCON, with probability at least $1-2M\delta-\delta$, we obtain
\begin{align*}
      R(T) &\leq C \,\,  \max\{\frac{\log T}{\Delta^2}, \log (|\mathcal{F}_{M}|), \log (1/\delta)\} \\
      &+ \mathcal{O} \left(\sqrt{KT \log(|\mathcal{F}_{d^*}|T/\delta)}\right).
      \end{align*}
\end{theorem}

\begin{remark}[Matches Oracle]

The first term of the regret scales weakly with $T$ (as $\mathcal{O}(\frac\log T/\Delta^2)$). Hence, provided  $\Delta^2 \geq \frac{\log T}{\sqrt{KT \log(|\mathcal{F}_{d^*}|T/\delta)}}$, the regret scaling (with respect to $T$) is dominated by $R_{\cb(\mathcal{F}_{d^*})}(T)$ (in case of FALCON, this term is $\Tilde{\mathcal{O}}(\sqrt{KT \log(|\mathcal{F}_{d^*}|T/\delta)}$, with high probability). However note that this is the regret of an oracle knowing the true function class $\mathcal{F}_{d^*}$. 
\end{remark}
\begin{remark}[Model selection Cost]
The first term can be interpreted as the cost of model selection and it depends on the gap $\Delta$. Hence, the model selection procedure only adds a $\mathcal{O}(\frac{\log T}{\Delta^2})$ term (this term is minor in the regime $\Delta^2 \geq \frac{\log T}{\sqrt{KT \log(|\mathcal{F}_{d^*}|T/\delta)}}$) term compared to the $\sqrt{T}$ scaling).
\end{remark}
\begin{remark} [Adaptive]
Algorithm \ref{algo:falcon} does not require knowledge of $\Delta$. Nevertheless, the regret guarantee adapts to the problem hardness, i.e., if $\Delta$ is small, the regret is larger and vice-versa.
\end{remark}
\begin{remark}[Improvement from $\mathcal{O}(\log T)$ to $\mathcal{O}(\log \log T)$ in the model selection cost]
We emphasize that the $\mathcal{O}(\log T)$ factor in the cost of model selection term can be improved, if we have the knowledge of $T$ apriori. In that setting, instead of substituting $\delta_m = \delta/2^m$, we substitute $\delta_m = \delta/\log T$ for all $m$. Since the doubling epoch ensures a total of $\mathcal{O}(\log T)$ epochs, this choice of $\delta_m$ yields
{\color{black}\begin{align*}
        R(T) &\leq  C \max\{\frac{1}{\Delta^2}, \log (|\mathcal{F}_{M}|), \log (\log T/\delta)\} \\
        &+ R_{\cb(\mathcal{F}_{d^*})}(T),
\end{align*}}
with probability at least $1-2M\delta$.
\end{remark}
\begin{remark} [Stronger Oracle in \cite{krishnamurthy2021optimal}]
\label{rem:strong_asm}
The cost of model selection in Theorem \ref{thm:falcon}, depends on the complexity of the largest model class $\mathcal{F}_M$. Under a stronger assumption on the regression oracle (for example Assumption $2$ of \cite{krishnamurthy2021optimal}), the cost of model selection can only depend on $\mathcal{F}_{d^*}$ as opposed to $\mathcal{F}_M$. Based on samples obtained from pure exploration for a realizable function class, we use \cite{agarwal2012contextual} to bound the excess risk (i.e., $\mathbb{E}(\hat{f}-f^*_{d^*})^2$) as a function of $\log(|\mathcal{F}_i|)$. In particular, since $\mathcal{F}_M$ (the largest class) is always realizable, we obtain an upper bound dependent on $\log(\mathcal{F}_M)$. On the other hand, Assumption 2 of \cite{krishnamurthy2021optimal} leads to an upper bound dependent on $\log(|\mathcal{F}_{d^*}|)$ only (since they take a minimum over all realizable classes).
\end{remark}


\section{Generic contextual bandits with infinite function classes} 
\label{sec:general_infinite}
The results in Section~\ref{sec:general} hold for finite function classes, since the regret bound depends on the cardinality of the function class. However, it can be extended to the infinite function classes (see \cite{falcon} for details). Exploiting the notion of the complexity of infinite function classes, this reduction is done.

Like before, we consider a nested sequence of $M$ function classes $\mathcal{F}_1 \subset \ldots \subset \mathcal{F}_M$. The reward is sampled from an unknown function $f^*_{d^*}$ lying in the (smallest) function class indexed by $d^* \in [M]$, which is unknown. Given the function classes, our job is to find the function class $\mathcal{F}_{d^*}$, and subsequently exploit the class to obtain sub-linear regret. Let us first rewrite the separability assumption.

We assume that the function classes $\mathcal{F}_1 \subset \ldots \subset \mathcal{F}_M$ are compact. This, in conjunction with the extreme value theorem, it is ensured that the following minimizers exist: for $j < d^*$, we define
\begin{align*}
    \Bar{f}_j = \mathrm{arginf}_{f \in \mathcal{F}_j} \,\, \mathbb{E}_{x,a}[f(x,a)-f^*_{d^*}(x,a)]^2
\end{align*}
for all pairs $(x,a)$. For $j \geq d^*$, we know that this minimizer is indeed $f^*_{d^*}$. This comes directly from the realizibility assumption. Note that we require the existence of the minimizer (regression function) in order to use it for selecting actions in the contextual bandit framework (see \cite{falcon})

Having defined the minimizers, we rewrite the separability assumption as following:
\begin{assumption}
\label{asm:sep_inf}
For any $\Bar{f}_j$, where $j<d^*$, we have
\begin{align*}
    \mathbb{E}_{x}\left [\inf_{a \in \mathcal{A}} (\Bar{f}_j(x,a) - f^*_{d^*}(x,a))^2\right] \geq \Delta.
\end{align*}
\end{assumption}

Similar to \cite{falcon}, here, we are not worried about the explicit form of the regression functions $\Bar{f}_j$. Rather, we assume the following performance guarantee of the offline regressor. For $j \geq d^*$ (meaning, the class containing the true regressor  $f^*_{d^*}$), we have the following assumption.
\begin{assumption}
\label{asm:pred_error_inf}
Given $n$ i.i.d data samples $(x_1,a_1,r_1(a_1)),(x_2,a_2,r_2(a_2)),\ldots,(x_n,a_n,r_n(a_n))$, the offline regression oracle returns a function $\hat{f}_j$, such that for $\delta>0$, with probability at least $1-\delta$,
\begin{align*}
    \mathbb{E}_{x,a}[\hat{f}_j(x,a)-f^*_{d^*}(x,a)]^2 \leq \xi_{\mathcal{F}_j,\delta}(n)
\end{align*}
\end{assumption}
This assumption is taken from \cite[Assumption 2]{falcon}. As discussed in the above-mentioned paper, the quantity $\xi_{(.,.)}(n)$ is a decreasing function of $n$, e.g., $\xi_{(.,.)}(n) = \Tilde{\mathcal{O}}(1/n)$. As an instance, consider the class of all linear regressors in $\mathbb{R}^d$. In that case, $\xi_{(.,.)}(n) \sim \Tilde{\mathcal{O}}(d/n)$. For function classes with finite VC dimension (or related quantities like VC-sub graph or fat-shattering dimension; pseudo dimension in general, denoted by $\Tilde{d}$), we have $\xi_{(.,.)}(n) \sim \Tilde{\mathcal{O}}(\Tilde{d}/n)$. 

In this section, we consider:
\begin{enumerate}
    \item The \texttt{ETC} algorithm (Algorithm~\ref{algo:main_algo}) with $\cb=$ FALCON
    \item The adaptive contextual bandit (\texttt{ACB}) algorithm (Algorithm~\ref{algo:falcon})  with $\cb=$ FALCON
\end{enumerate}

The model-selection algorithm remains the same. For Option I, we explore for the first $2\sqrt{T}$ rounds. The first $\sqrt{T}$ rounds are used to collect samples $(x_t,r_t,a_t)$ via pure exploration. Feeding this samples to the offline regression oracle, and focusing on the individual function classes $\{\mathcal{F}_j\}_{j=1}^M$ separately, we obtain $(\hat{f}_j, \xi_{\mathcal{F}_j,\delta}(\sqrt{T}))$ for all $j \in [M]$. Thereafter, we perform another round of pure exploration, and obtain $\sqrt{T}$ fresh samples. Like in the finite case, we construct statistic $S_j$ for all $j \in [M]$.

For Option II, we collect all the samples from the previous epoch of the FALCON algorithm, split the samples, to obtain the regression estimate $\hat{f}^m_j$ and similarly construct test statistic $S^m_j$ for all $j \in [M]$. In this setting, for the $m$-th epoch, with model chosen as $\mathcal{F}_\ell$, we set the learning rate (similar to the FALCON$+$ algorithm of \cite{falcon}) as
\begin{align*}
    \rho_m = (1/30) \sqrt{K/\xi_{\mathcal{F}_\ell^m,\delta/2m^2}(\tau_{m-1}-\tau_{m-2})}.
\end{align*}

Similar to Algorithms~\ref{algo:main_algo} and \ref{algo:falcon}, we choose the correct model based on a threshold on the test statistic $S_j^m$ (for Option II, it is $S_j$) and the threshold in phase $m$ is $\gamma^m:= S^m_M\sqrt{\frac{m}{2^m}}$ ( $\gamma:= S_M+\sqrt{\frac{\log T}{\sqrt{T}}}$ for Option II). We show that for all sufficiently large phase numbers, for all $j \geq d^*$, $S_j^m \leq \gamma^m$, and for all $j < d^*$, $S_j^m > \gamma^m$ with high probability.  Once this is shown, the model selection procedure follows exactly as Algorithm~\ref{algo:main_algo}, i.e., we find the smallest index $\ell \in [M]$, for which $S_{\ell} \leq \gamma^m$. With high probability, we show that $\ell=d^*$.

\paragraph{Regret Guarantee}
We first show the guarantees for Option I, and  Option II.

\begin{theorem}
\label{thm:etc_inf}
(\texttt{ETC} with $\cb$ = FALCON) Suppose Assumptions~\ref{asm:second_moment}, \ref{asm:sep_inf} and \ref{asm:pred_error_inf} hold. Then, provided,
\begin{align*}
    T \gtrsim (\log T) \max \left(  T^{1/4} \xi_{\mathcal{F}_M, (1/T^{1/4})} , \Delta^{-4}, \log(1/\delta) \right),
\end{align*}
with probability at least $1-4M\delta$, line $11$ in Algorithm~\ref{algo:main_algo} identifies the correct model class $\mathcal{F}_{d^*}$. Furthermore, running Algorithm~\ref{algo:main_algo} for $T$ iterations yields, with probability at least $1-2M\delta -\delta$, the regret
\begin{align*}
    R(T) \leq C \sqrt{T} + \mathcal{O} \left (\sqrt{K \xi_{\mathcal{F}_{d^*},\delta/2T}(T)} \,\, T \right).
\end{align*}
\end{theorem}

\begin{theorem}
\label{thm:falcon_inf}
(\texttt{ACB} with $\cb$ = FALCON) Suppose Assumptions~\ref{asm:second_moment}, \ref{asm:sep_inf} and \ref{asm:pred_error_inf} hold. Then, with probability at least $1-2M\delta-\delta$, running Algorithm~\ref{algo:falcon} for $T$ iterations yield
\begin{align*}
    R(T) \leq C & (\log T)  \max\{ \max_m  2^{m/2} \, \xi_{\mathcal{F}_M,1/2^{m/2}}(2^{m-2}), \\
    & \log (1/\delta), \Delta^{-2} \} + \mathcal{O}\left (\sqrt{K \xi_{\mathcal{F}_{d^*},\delta/2T}(T)} \, T \right).
\end{align*}
\end{theorem}

\begin{remark}[Matching Oracle regret]
In both the settings, we match the regret of an oracle knowing the correct function class (see \cite{falcon}). We pay a small additive price for model selection.
\end{remark}
\begin{remark}
The proof of these theorems parallels exactly similar to the finite function class setting. The only difference is that instead of upper-bounding the prediction error using technical tools from \cite{agarwal2012contextual}, we use the the definition of $\xi(.)$ to accomplish this. 
\end{remark}

\section{Model Selection in Stochastic Linear Bandits}
\label{sec:linear}
In the previous sections, we consider the problem of model selection for general contextual bandits. Moreover, we assumed that the function classes are separable, and leveraging that we have several provable model selection algorithms. In this section, we consider a special case of model selection for stochastic linear bandits. We observe that with this linear structure, assumption like separability across function classes is not required.

In the linear bandit settings, we consider $2$ different setup---(a) continuum (infinite) arm setting and (b) finite arm setting. We first start with the continuum arm setup.

\subsection{Model Selection for Continuum (infinite) Arm Stochastic Linear bandits}
\label{sec:dimension_adaptation}

\subsubsection{Setup}
We consider the standard stochastic linear bandit model in $d$ dimensions (see \cite{oful}), with the dimension as a measure of complexity. The setup comprises of a continuum collection of arms denoted by the set $\mathcal{A}:= \{x \in \mathbb{R}^d:\|x\| \leq 1\}$\footnote{Our algorithm can be applied to any compact set $\mathcal{A} \subset \mathbb{R}^d$, including the finite set as shown in Appendix \ref{appendix-comparision}.}
Thus, the mean reward from any arm $x \in \mathcal{A}$ is $ \langle x,\theta^* \rangle $, where $\|\theta^*\| \leq 1$. 
We assume that $\theta^*$ is $d^* \leq d$ sparse, where $d^*$ is apriori unknown to the algorithm.
For each time $t \in [T]$, if an algorithm chooses an arm  $x_t \in \mathcal{A}$, the observed reward is denoted by $y_t:= \langle x_t,\theta^* \rangle + \eta_t$, where $\{\eta_t\}_{t\geq 1}$ is an i.i.d. sequence of $0$ mean sub-gaussian random variables with known parameter $\sigma^2$.

We consider a sequence of $d$ nested hypothesis classes, where each hypothesis class $i \leq d$, models $\theta^*$ as a $i$ sparse vector. The goal of the forecaster is to minimize the regret, namely 
\begin{align*}
    R(T) = \sum_{t=1}^T \left[ \langle x^*_t - x_t,\theta^* \rangle\right],
\end{align*}
where at any time $t$, $x_t$ is the action recommended by an algorithm and $x^*_t = \mathrm{argmax}_{x \in \mathcal{A}} \langle x,\theta^*\rangle$. The regret $R(T)$ measures the loss in reward of the forecaster with that of an oracle that knows $\theta^*$ and thus can compute $x^*_t$ at each time.

Note that, we assume that the \emph{true complexity} (dimension) $d^* \leq d$ is initially {unknown}, and we seek algorithms that adapts to this unknown true dimension, rather than assume that the problem is $d$ dimensional. This is in contrast to both the standard linear bandit setup \cite{chu2011contextual,oful}, where there is no notion of complexity, as well as the line of work on sparse linear bandits \cite{sparse_bandit2}, where the \emph{the true sparsity (dimension)} is known, but only the set of which of the $d^*$ out of the $d$ coordinates is non-zero is unknown.

\subsubsection{Algorithm: Adaptive Linear Bandits (Dimension) [{\ttfamily ALB-Dim}] }


We present our adaptive scheme in Algorithm~\ref{algo:main_algo_dimensions_unknown}. The algorithm is parametrized by $T_0 \in \mathbb{N}$, which is given in Equation (\ref{eqn:T_0_defn}) in the sequel and slack $\delta \in (0,1)$.
 {\ttfamily ALB-Dim} proceeds in phases numbered $0,1,\cdots$ which are non-decreasing with time. 
At the beginning of each phase, {\ttfamily ALB-Dim} makes an estimate of the set of non-zero coordinates of $\theta^*$, which is kept fixed throughout the phase.
Concretely, each phase $i$ is divided into two blocks:
\begin{enumerate}
    \item a regret minimization block lasting $36^i T_0$ time slots\footnote{We have not optimized over the constants like $36$ and $6$. Please refer to Remark~\ref{rem:cons} on this.},
    \item followed by a random exploration phase lasting $6^i \lceil\sqrt{T_0}\rceil$ time slots.
\end{enumerate}
Thus, each phase $i$ lasts for a total of $36^iT_0 + 6^i \lceil \sqrt{T_0} \rceil$ time slots.
At the beginning of each phase $i \geq 0$, $\mathcal{D}_i \subseteq [d]$ denotes the set of `active coordinates', namely the estimate of the non-zero coordinates of $\theta^*$.
By notation, $\mathcal{D}_0 = [d]$ and at the start of phase $0$, the algorithm assumes that $\theta^*$ is $d$ sparse.
Subsequently, in the regret minimization block of phase $i$, a fresh instance of OFUL \cite{oful} is spawned, with the dimensions restricted only to the set $\mathcal{D}_i$ and probability parameter $\delta_i:= \frac{\delta}{2^i}$. In the random exploration phase, at each time, one of the possible arms from the set $\mathcal{A}$ is played chosen uniformly and independently at random. 
At the end of each phase $i\geq 0$,  {\ttfamily ALB-Dim} forms an estimate $\widehat{\theta}_{i+1}$ of $\theta^*$, by solving a least squares problem using all the random exploration samples collected till the end of phase $i$.
The active coordinate set $\mathcal{D}_{i+1}$, is then the coordinates of $\widehat{\theta}_{i+1}$ with  magnitude exceeding $2^{-(i+1)}$.
The pseudo-code is provided in Algorithm \ref{algo:main_algo_dimensions_unknown}, where, $\forall i \geq 0$, $S_i$ in lines $15$ and $16$ is the total number of random-exploration samples in all phases upto and including $i$.

\begin{algorithm}[t!]
  \caption{Adaptive Linear Bandit (Dimension)}
  \begin{algorithmic}[1]
 \STATE  \textbf{Input:} Initial Phase length $T_0$ and slack $\delta > 0$.
 \STATE $\widehat{\theta}_0 = \mathbf{1}$, $T_{-1}=0$
 \FOR {Each epoch $i \in \{0,1,2,\cdots\}$}
 \STATE $T_i = 36^{i} T_0$, $\quad$  $\varepsilon_i \gets \frac{1}{2^{i}}$, $\quad$  $\delta_i \gets \frac{\delta}{2^{i}}$
 \STATE $\mathcal{D}_i := \{i : |\widehat{\theta}_i| \geq \frac{\varepsilon_i}{2} \}$
 \FOR {Times $t \in \{T_{i-1}+1,\cdots,T_i\}$}
 \STATE Play $\text{OFUL}(1,\delta_i)$ only restricted to coordinates in $\mathcal{D}_i$. Here $\delta_i$ is the probability slack parameter and $1$ represents $\|\theta^*\| \leq 1$.
 \ENDFOR
 \FOR {Times $t \in \{T_i+1,\cdots,T_i + 6^i\sqrt{T_0}\}$}
 \STATE Play an arm from the action set $\mathcal{A}$ chosen uniformly and independently at random.
 \ENDFOR
 \STATE $\boldsymbol{\alpha}_i \in \real^{S_i \times d}$ with each row being  the arm played during all random explorations in the past.
 \STATE $\boldsymbol{y}_i \in \real^{S_i}$  with $i$-th entry being the observed reward at the $i$-th random exploration in the past
 \STATE $\widehat{\theta}_{i+1} \gets (\boldsymbol{\alpha}_i^T\boldsymbol{\alpha}_i)^{-1}\boldsymbol{\alpha}_i\mathbf{y}_i$, is a $d$ dimensional vector
 \ENDFOR
  \end{algorithmic}
  \label{algo:main_algo_dimensions_unknown}
\end{algorithm}

\subsubsection{Regret Guarantee}

We first specify, how to set the input parameter $T_0$, as function of $\delta$.
For any $N \geq d$, denote by $A_N$ to be the $N \times d$ random matrix with each row being a vector sampled uniformly and independently from the unit sphere in $d$ dimensions.
Denote by $M_N := \frac{1}{N} \mathbb{E}[A_N^TA_N]$, and by $\lambda_{\max}^{(N)},\lambda_{\min}^{(N)}$, to be the largest and smallest eigenvalues of $M_N$. Observe that as $M_N$ is positive semi-definite ($0 \leq \lambda_{\min}^{(N)}\leq\lambda_{\max}^{(N)}$) and almost-surely full rank, i.e., $\mathbb{P}[\lambda_{\min}^{(N)} > 0] = 1$.
The constant $T_0$ is the smallest integer such that
\begin{align}
    \sqrt{T_0} \geq & \max \bigg ( \frac{32\sigma^2}{(\lambda_{min}^{(\lceil \sqrt{T_0} \rceil)})^2}\ln (2d/\delta), \notag \\
    &\frac{4}{3} \frac{(6\lambda_{max}^{(\lceil \sqrt{T_0}\rceil)}+\lambda_{min}^{(\lceil \sqrt{T_0}\rceil)})(d+\lambda_{max}^{(\lceil \sqrt{T_0}\rceil)})}{(\lambda_{min}^{(\lceil \sqrt{T_0}\rceil)})^2}\ln ( 2d/\delta) \bigg)
    \label{eqn:T_0_defn}
\end{align}

\begin{remark}
$T_0$ in Equation (\ref{eqn:T_0_defn}) is chosen such that, at the end of phase $0$, 
$
    \mathbb{P}[||\widehat{\theta}_0 - \theta^*||_{\infty} \geq 1/2 ] \leq \delta
$ \cite{linear_reg_guarantees}.
A formal statement of the Remark is provided in Lemma \ref{lem:reg_bounds} in Appendix \ref{sec:proofs}.
\end{remark}

\begin{theorem}
Suppose Algorithm \ref{algo:main_algo_dimensions_unknown} is run with input parameters $\delta \in (0,1)$, and $T_0$ as given in Equation (\ref{eqn:T_0_defn}), then with probability at-least $1-\delta$, the regret after a total of $T$ arm-pulls satisfies
\begin{align*}
R_T &\leq C \frac{T_0}{{\gamma^{5.18}}}T_0
+ C_1 \sqrt{T} \bigg [ 1 +   \sqrt{d^*\ln ( 1 + \frac{T}{d^*} )} \\
& \times (1 + \sigma\sqrt{ \ln ( \frac{T}{T_0\delta} ) + d^* \ln ( 1+\frac{T}{d^*})})\bigg ].
\end{align*}
The parameter $\gamma > 0$ is the minimum  magnitude of the non-zero coordinate of $\theta^*$, i.e., $\gamma = \min \{|\theta^*_i| : \theta^*_i \neq 0 \}$ and $d^*$ the sparsity of $\theta^*$, i.e., $d^* = |\{i:\theta^*_i\neq 0\}|$.
\label{thm:adaptive_dimension}
\end{theorem}

In order to parse this result, we give the following corollary.
\begin{corollary}
Suppose Algorithm \ref{algo:main_algo_dimensions_unknown} is run with input parameters $\delta \in (0,1)$, and $T_0 = \widetilde{O} \left(d^2\ln^2 \left( \frac{1}{\delta} \right) \right)$ given in Equation (\ref{eqn:T_0_defn}), then with probability at-least $1-\delta$, the regret after $T$ times satisfies
\begin{align*}
        R_T &\leq O ( \frac{d^2}{{\gamma^{5.18}}} \ln^2 ( d/\delta) )  + \widetilde{O} ( d^* \sqrt{ T}).
\end{align*}
\label{cor:dimension_adaptation}
\end{corollary}

\begin{remark}
\label{rem:cons}
The constants in the above Theorem are not optimized. The epoch length and the threshold parameter $\varepsilon_i$  can be chosen more carefully. For example, if we set the epoch length as $4^i T_0 + 2^i \sqrt{T_0}$ and the threshold $\varepsilon_i$ as $(0.9)^i$, we obtain a worse dependence on $\gamma$. Furthermore, the exponent of $\gamma$ can be made arbitrarily close to $4$, by setting $\varepsilon_i = C^{-i}$ in Line $4$ of Algorithm \ref{algo:main_algo_dimensions_unknown}, for some appropriately large constant $C > 1$, and increasing $T_i = (C')^iT_0$, for appropriately large $C'$ ($C'\approx C^4)$.
\end{remark}
\begin{remark}
    In this special case of linear bandits, the separability condition boils down to $\gamma > 0$, which comes with the problem setup automatically, since the complexity of the problem is the number of non-zero entries of the underlying true parameter $\theta^*$.
\end{remark}

\noindent {\textbf{Discussion - }}
The regret of an oracle algorithm that knows the true complexity $d^*$ scales as $\widetilde{O}(d^*\sqrt{T})$ \cite{sparse_bandit1,sparse_bandit2}, matching  {\ttfamily ALB-Dim}'s regret, upto an additive constant independent of time.
{\ttfamily ALB-Dim} is the first algorithm to achieve such model selection guarantees.
On the other hand, standard linear bandit algorithms such as {\ttfamily OFUL} achieve a regret scaling $\widetilde{O}(d\sqrt{T})$, which is much larger compared to that of {\ttfamily ALB-Dim}, especially when $d^* << d$, and $\gamma$ is a constant.
Numerical simulations further confirms this deduction, thereby indicating that our improvements are fundamental and not from mathematical bounds.
 Corollary \ref{cor:dimension_adaptation} also indicates that {\ttfamily ALB-Dim} has higher regret if $\gamma$ is lower. A small value of $\gamma$ makes it harder to distinguish a non-zero coordinate from a zero coordinate, which is reflected in the regret scaling. 
Nevertheless, this only affects the \emph{second order term as a constant}, and the dominant scaling term only depends on the true complexity $d^*$, and not on the underlying dimension $d$.
However, the regret guarantee is not uniform over all $\theta^*$ as it depends on $\gamma$. Obtaining regret rates matching the oracles and that hold uniformly over all $\theta^*$ is an interesting avenue of future work.

\subsection{Dimension as a Measure of Complexity - Finite Armed Setting}
\label{sec:dim_adap_finite}

\subsubsection{Setup} 
In this section, we consider the model selection problem for the setting with finitely many arms in the framework studied in \cite{foster_model_selection}. 
At each time $t \in [T]$, the forecaster is shown a context $X_t \in \mathcal{X}$, where $\mathcal{X}$ is some arbitrary `feature space'. The set of contexts $(X_t)_{t=1}^T$ are i.i.d. with $X_t \sim \mathcal{D}$, a probability distribution over $\mathcal{X}$ that is known to the forecaster.
Subsequently, the forecaster chooses an action $A_t \in \mathcal{A}$, where the set $\mathcal{A} \coloneqq \{1,\cdots,K\}$ are the $K$ possible actions chosen by the forecaster. The forecaster then receives a reward $Y_t := \langle \theta^*, \phi^M(X_t,A_t) \rangle + \eta_t$. Here $(\eta_t)_{t =1}^T$ is an i.i.d. sequence of $0$ mean sub-gaussian random variables with sub-gaussian parameter $\sigma^2$ that is known to the forecaster.
The function\footnote{Superscript $M$ will become clear shortly} $\phi^M : \mathcal{X} \times \mathcal{A} \rightarrow \mathbb{R}^d$ is a known feature map, and $\theta^* \in \mathbb{R}^d$ is an unknown vector. 
The goal of the forecaster is to minimize its regret, namely $R(T) \coloneqq \sum_{t=1}^T \mathbb{E}\left[ \langle A^*_t - A_t,\theta^* \rangle\right]$, where at any time $t$, conditional on the context $X_t$, $A^*_t \in \argmax_{a \in \mathcal{A}} \langle \theta^*,\phi^M(X_t,a) \rangle$. Thus, $A^*_t$ is a random variable as $X_t$ is random.

To describe the model selection, we consider a sequence of $M$ dimensions $1 \leq d_1 < d_2,\cdots < d_M \coloneqq d$ and an associated set of feature maps $(\phi^m)_{m=1}^M$, where for any $m \in [M]$, $\phi^m(\cdot,\cdot) : \mathcal{X} \times \mathcal{A} \rightarrow \mathbb{R}^{d_m}$, is a feature map embedding into $d_m$ dimensions. Moreover, these feature maps are nested, namely, for all $m \in [M-1]$, for all $x \in \mathcal{X}$ and $a \in \mathcal{A}$, the first $d_m$ coordinates of $\phi^{m+1}(x,a)$ equals $\phi^m(x,a)$. The forecaster is assumed to have knowledge of these feature maps.
The unknown vector $\theta^*$ is such that its first $d_{m^*}$ coordinates are non-zero, while the rest are $0$. 
The forecaster does not know the true dimension $d_{m^*}$.
 If this were known, than standard contextual bandit algorithms such as LinUCB \cite{chu2011contextual} can guarantee a regret scaling as $\widetilde{O}(\sqrt{d_{m^*}T})$.  In this section, we provide an algorithm in which, even when the forecaster is unaware of $d_{m^*}$, the regret scales as $\widetilde{O}(\sqrt{d_{m^*}T})$. However, this result is non uniform over all $\theta^*$ as, we will show, depends on the minimum non-zero coordinate value in $\theta^*$.
\vspace{2mm}

\noindent\textbf{Model Assumptions} We will require some assumptions identical to the ones stated in \cite{foster_model_selection}. Let $\|\theta^*\|_2 \leq 1$, which is known to the forecaster. The distribution $\mathcal{D}$  is assumed to be known to the forecaster. Associated with  the distribution $\mathcal{D}$
 is a matrix $\Sigma_M \coloneqq \frac{1}{K} \sum_{a \in \mathcal{A}} \mathbb{E} \left[ \phi^M(x,a) \phi^M(x,a)^T\right]$ (where $x \sim \mathcal{D}$), where we assume its minimum eigen value $\lambda_{min}(\Sigma_M)  > 0$ is strictly positive.
Further, we assume that, for all $a \in \mathcal{A}$, the random variable $\phi^M(x,a)$ (where $x \sim \mathcal{D}$ is random) is a sub-gaussian random variable with (known) parameter $\tau^2$.

\subsubsection{{\ttfamily ALB-Dim} Algorithm}

The algorithm here is identical to that of Algorithm \ref{algo:main_algo_dimensions_unknown}, except that in place of OFUL, we use {\ttfamily SupLinRel} of \cite{chu2011contextual} as the black-box. 
The details of the Algorithm are provided in Appendix \ref{appendix-comparision}.

\subsubsection{Regret Guarantee}
For brevity, we only state the Corollary of our main Theorem (Theorem \ref{thm:adaptive_dimension_foster}) which is stated in Appendix \ref{appendix-comparision}.

\begin{corollary}
Suppose Algorithm \ref{algo:main_algo_dimensions_foster} is run with input parameters $\delta \in (0,1)$, and $T_0 = \widetilde{O} \left(d^2\ln^2 \left( \frac{1}{\delta} \right) \right)$ given in Equation (\ref{eqn:T_0_defn_foster}) , then with probability at-least $1-\delta$, the regret after $T$ times satisfies
\begin{align*}
        R_T &\leq O \left( \frac{d^2}{{\gamma^{5.18}}} \ln^2 ( d/\delta) \tau^2 \ln \left( \frac{TK}{\delta}\right) \right)  + \widetilde{O} (  \sqrt{ T d^*_m}),
\end{align*}
where $\gamma = \min \{|\theta^*_i| : \theta^*_i \neq 0 \}$ and $\theta^*$ is $d^*$ sparse.
\label{cor:dimension_adaptation_foster}
\end{corollary}

\noindent {\textbf{Discussion - }}
Our regret scaling matches that of an oracle that knows the true problem complexity and thus obtains a regret of $\widetilde{O}(\sqrt{d_{m^*}T})$. This, thus improves on the rate compared to that obtained in \cite{foster_model_selection}, whose regret scaling is sub-optimal compared to the oracle. On the other hand however, our regret bound depends on $\gamma$ and is thus not uniform over all $\theta^*$, unlike  \cite{foster_model_selection} that is uniform over $\theta^*$. Thus, in general, our results are not directly comparable to that of \cite{foster_model_selection}. It is an interesting future work to close the gap and in particular, obtain the regret matching that of an oracle to hold uniformly over all $\theta^*$.

\vspace{-2mm}
\section{Conclusion}
\vspace{-1mm}
In this paper, we address the problem of model selection for generic contextual bandits. We propose and analyze a meta algorithm, that takes any provable base algorithm as blackbox and performs model selection on top. Moreover, we also analyze a much simpler algorithm based on explore and commit for model selection. Our model selection schemes rely on realizibility and separability assumptions, and remove (or weaken) them is an immediate future work. We would also like to work on model selection problems for Reinforcement Learning problems. We keep these as our future endeavors.


%
\IEEEpeerreviewmaketitle

\section*{Acknowledgements}
The authors would like to acknowledge Akshay Krishnamurthy, Dylan Foster and Haipeng Luo for insightful comments and suggestions.

\newpage
\section*{Appendix}
\section{Model Selection for Contextual Bandits}
\label{sec:proofs}

\subsection{Proof of Lemma~\ref{lem:etc}} Since, we have samples from pure exploration, let us first show that $S_j$ concentrates around its expectation. We show it via a simple application of the Hoeffdings inequality.
 
Fix a particular $j \in [M]$. Note that $\hat{f}_j$ is computed based on the first set of $\lceil \sqrt{T} \rceil$ samples. Also, in the testing phase, we again sample $\lceil \sqrt{T} \rceil$ samples, and so $\hat{f}$ is independent of the second set of $\lceil\sqrt{T}\rceil$ samples, used in constructing $S_j$. Note that since we have $r_t(.) \in [0,1]$, we may restrict the offline regression oracle to search over functions having range $[0,1]$. This implies that, we have $\hat{f}_j^m(.) \in [0,1]$. Note that this restricted search assumption is justified since our goal is obtain an estimate of the reward function via regression function, and this assumption also features in \cite{falcon}. So the random variable $(\hat{f}_j(x_t,a_t) -r_t(a_t))^2$ is upper-bounded by $4$, and hence sub-Gaussian with a constant parameter. Also, note that since we are choosing an action independent of the context, the random variables $\{(\hat{f}_j(x_t,a_t) -r_t(a_t))^2\}_{t=1}^{\lceil\sqrt{T}\rceil}$ are independent. Hence using Hoeffdings inequality for sub-Gaussian random variables, we have
\begin{align*}
    \Prob \left( | S_j - \mathbb{E} S_j | \geq \ell \right) \leq 2\exp (- n \ell^2/32).
\end{align*}
 Re-writing the above, we obtain
 \begin{align}
     |S_j - \mathbb{E}S_j| \leq C\sqrt{\frac{ \log(1/\delta)}{ \sqrt{T}}}
     \label{eqn:hoeffding}
 \end{align}
 with probability at least $1-2\delta$ with $\sqrt{T}$ samples. 
 
Note that, the conditional variance of $r_t(.)$ is finite, i.e., given $x_t = x \in \mathcal{X}$,
$\mathbb{E}[r_t(a) - f^*_{d^*}(x,a)]^2 \leq 1$, for all $a \in \mathcal{A}$. Let us define\footnote{We use the notation $\sigma^2$ throughout the rest of the paper.} $\mathbb{E}[r_t(a) - f^*_{d^*}(x,a)]^2 =\sigma^2$. To be concrete $\sigma$ depends on epoch $m$ and hence $\sigma_m$ makes more sense. We omit the subscript for notational simplicity. With this new notation, let us first look at the expression $\mathbb{E} S_j$.
 
Let us look at the expression $\mathbb{E} S_j$.
 \begin{align*}
     \mathbb{E}S_j = \mathbb{E} \left( \frac{1}{\lceil \sqrt{T} \rceil}\sum_{t=1}^{\lceil \sqrt{T} \rceil} (\hat{f}_j (x_t,a_t) - r_t(a_t))^2 \right).
 \end{align*}
\paragraph{Case I: Realizable Class} First consider the case that $j \geq d^*$, meaning that $f^*_{d^*} \in \mathcal{F}_j$. So, for this realizable setting, we obtain the excess risk as (using \cite{agarwal2012contextual})
\begin{align*}
    & \mathbb{E}_{x,r,a}[\hat{f}_j(x,a) - r(a)]^2 - \inf_{f \in \mathcal{F}_j} \mathbb{E}_{x,r,a}[f(x,a) -r(a)]^2 \\
    & = \mathbb{E}_{x,r,a}[\hat{f}_j(x,a) - r(a)]^2 - \mathbb{E}_{x,r,a} [f^*_{d^*} (x,a) - r(a)]^2 \\
    & = \mathbb{E}_{x,a} [\hat{f}_j(x,a) - f^*_{d^*}(x,a)]^2.
\end{align*}
So, we have, for the realizable function class,
\begin{align*}
   & \mathbb{E}S_j = \frac{1}{\lceil \sqrt{T} \rceil} \mathbb{E}_{x_t,r_t,a_t}\sum_{t=1}^{\lceil \sqrt{T} \rceil}[\hat{f}_j(x_t,a_t) - r_t(a_t)]^2 \\ 
   & = \frac{1}{\lceil \sqrt{T} \rceil}\sum_{t=1}^{\lceil \sqrt{T} \rceil}\mathbb{E}_{x_t,r_t,a_t} [f^*_{d^*} (x_t,a_t) - r_t(a_t)]^2 \\
   & + \frac{1}{\lceil \sqrt{T} \rceil} \sum_{t=1}^{\lceil \sqrt{T} \rceil}\mathbb{E}_{x_t,a_t} [\hat{f}(x,a) - f^*_{d^*}(x,a)]^2 \\
    & {\color{black} \leq \sigma^2 + C_1 \frac{\log (\sqrt{T}|\mathcal{F}_j|)}{\sqrt{T}}}
\end{align*}
where $C_1$ is an absolute constant. The second term is obtained by setting the high probability slack, as $2^{-m/2}$ into  \cite[Lemma 4.1]{agarwal2012contextual}. So, we finally have from the preceeding display and Equation (\ref{eqn:hoeffding}) that
\begin{align}
  \sigma^2  - C_2  \sqrt{\frac{\log(1/\delta)}{\sqrt{T}}} \leq  S_j & \leq \sigma^2 + C_2 \frac{\log (\sqrt{T}|\mathcal{F}_j|)}{\sqrt{T}} \notag \\
  & + C_3  \sqrt{\frac{\log(1/\delta)}{\sqrt{T}}}
    \label{eqn:etc_realizable_fin}
\end{align}
with probability at least $1-2\delta$.

\paragraph{Case II: Non-realizable class} We now consider the case when $j < d^*$, meaning that $f^*_{d^*}$ does not lie in $\mathcal{F}_j$. We have
\begin{align*}
   & \mathbb{E}_{x,r,a}[f(x,a)-r(a)]^2 - \mathbb{E}_{x,r,a}[r(a) - f^*_{d^*}(x,a)]^2 \\
   & = \E_{x,a,r}[(f(x,a) - f^*_{d^*}(x,a))(f(x,a) + f^*_{d^*}(x,a) - 2r(a)] \\
   & = \E_{x,a}\E_{r|x} [(f(x,a) - f^*_{d^*}(x,a))(f(x,a) + f^*_{d^*}(x,a) - 2r(a)] \\
   & = \E_{x,a} [(f(x,a) - f^*_{d^*}(x,a))(f(x,a) + f^*_{d^*}(x,a) - 2 \E_{r|x} r(a)] \\
   & = \E_{x,a}[f(x,a)- f^*_{d^*}(x,a)]^2,
\end{align*}
where the third inequality follows from the fact that given context $x$, the distribution of $r$ in independent of $a$ (see \cite[Lemma 4.1]{agarwal2012contextual}). Hence,
\begin{align*}
    \mathbb{E}_{x,r,a}[f(x,a)-r(a)]^2 & \geq \mathbb{E}_{x,r,a}[r(a) - f^*_{d^*}(x,a)]^2 \\
    & + \E_{x,a}[f(x,a)- f^*_{d^*}(x,a)]^2 \\
    & \geq \Delta + \sigma^2,
\end{align*}
where the last inequality comes from the separability assumption along with the definition of $\sigma$. Since the regressor $\hat{f}_j \in \mathcal{F}_j$, we have
\begin{align*}
    \mathbb{E}_{x,r,a}[\hat{f}_j(x,a)-r(a)]^2 & \geq \mathbb{E}_{x,r,a}[r(a) - f^*_{d^*}(x,a)]^2 \\
    &+ \E_{x,a}[f(x,a)- f^*_{d^*}(x,a)]^2 \\
    & \geq \Delta + \sigma^2,
\end{align*}
and hence
\begin{align*}
     \mathbb{E}S_j \geq \Delta + \sigma^2
\end{align*}

So, in this setting, with probability $1-2\delta$,
\begin{align}
    S_j & \geq \mathbb{E}S_j - C_4\sqrt{\frac{ \log(1/\delta)}{ \sqrt{T}}}\\
    & \geq \Delta + \sigma^2 - C_4\sqrt{\frac{\log(1/\delta)}{ \sqrt{T}}}.
    \label{eqn:etc_non_realizable_1}
\end{align}
where $C$ is an absolute global constant. Thus, from Equations (\ref{eqn:etc_realizable_fin}) and (\ref{eqn:etc_non_realizable_1}) and an union bound over the $M$ classes, we have with probability at-least $1-4M\delta$, 
\begin{equation}
\begin{aligned}
S_j &\geq \sigma^2 -  C_2 \frac{\log (\sqrt{T}|\mathcal{F}_j|)}{\sqrt{T}} - C_3  \sqrt{\frac{\log(1/\delta)}{\sqrt{T}}} , \text{ for all } j \geq d^{*},\\
  S_j &\leq \sigma^2 +  C_2 \frac{\log (\sqrt{T}|\mathcal{F}_j|)}{\sqrt{T}} + C_3  \sqrt{\frac{\log(1/\delta)}{\sqrt{T}}}, \text{ for all } j \geq d^{*},\\
    S_j &\geq \Delta +\sigma^2 - C_4\sqrt{\frac{\log(1/\delta)}{\sqrt{T}}}, \text{ for all } j < d^{*}.
\end{aligned}
\label{eqn:etc_good_event}
\end{equation}
{\color{black}\paragraph{Choice of Threshold} Notice from Line $11$ of Algorithm \ref{algo:main_algo}, that the threshold for model selection is $\gamma:= S_M + \sqrt{\frac{\log(T)}{\sqrt{T}}}$. Thus, if the event in Equations (\ref{eqn:etc_good_event}) holds, then the model selection stage will succeed in identifying the correct model class if the threshold $\gamma$ satisfies
\begin{equation}
    \begin{aligned}
    \gamma &< \Delta + \sigma^2 - C_4\sqrt{\frac{\log(1/\delta)}{ \sqrt{T}}}, \\
    \gamma &> \sigma^2 + C_2  \frac{\log (\sqrt{T}|\mathcal{F}_j|)}{\sqrt{T}} +C_3\sqrt{\frac{\log(1/\delta)}{\sqrt{T}}}
    \end{aligned}
    \label{eqn:etc_threshold_conditions}
\end{equation}
The first item ensures that no-non realizable class will be selected as the true model, and the second item ensures that the smallest realizable class will be selected as the true model. Thus, if the time horizon $T$ satisfies
\begin{align}
     \sqrt{\frac{\log(T)}{\sqrt{T}}} \geq 2 \left(C_2  \frac{\log (\sqrt{T}|\mathcal{F}_j|)}{\sqrt{T}} +C_3\sqrt{\frac{\log(1/\delta)}{\sqrt{T}}} \right), \\
   \sqrt{\frac{\log(T)}{\sqrt{T}}} + C_2  \frac{\log (\sqrt{T}|\mathcal{F}_j|)}{\sqrt{T}} +C_3\sqrt{\frac{\log(1/\delta)}{\sqrt{T}}}  \leq \Delta \notag \\
    - C_4\sqrt{\frac{ \log(1/\delta)}{ \sqrt{T}}},
    \label{eqn:etc_thresh_time_horizon}
\end{align}
then the threshold $\gamma$ satisfies the conditions in Equations (\ref{eqn:etc_threshold_conditions}). It is easy to verify that for $$T \gtrsim (\log T) \max \left( \log \left( \sqrt{T}{|\mathcal{F}_M|} \right), \Delta^{-4}, \log(1/\delta) \right),$$
the conditions in Equations (\ref{eqn:etc_thresh_time_horizon}) holds. Thus, Equations (\ref{eqn:etc_good_event}), (\ref{eqn:etc_threshold_conditions}) and (\ref{eqn:etc_thresh_time_horizon}) yield that, if  

$$T \gtrsim (\log T) \max \left( \log \left( \sqrt{T}{|\mathcal{F}_M|} \right), \Delta^{-4}, \log(1/\delta) \right),$$
with probability at-least $1-4M\delta$, the model selection test in Line $11$ of Algorithm \ref{algo:main_algo} correctly identifies the smallest model class containing the true model.
}

\subsection{Proof of Theorem~\ref{thm:etc}}
The regret $R(T)$ can be decomposed in $2$ stages, namely exploration and exploitation.
\begin{align*}
    R(T) = R_{explore} + R_{exploit}
\end{align*}
Since we spend $2\lceil\sqrt{T}\rceil$ time steps in exploration, and $r_t(.) \in [0,1]$, the regret incurred in this stage
\begin{align*}
    R_{explore} \leq C_1 \sqrt{T}.
\end{align*}
Now, at the end of the explore stage, provided Assumptions 2 and 3, we know, with probability at least $1-4M\delta$, we obtain the true function class $\mathcal{F}_{d^*}$. The threshold is set in such a way that we obtain the above result. Now, we would just commit to the function class and use the contextual bandit algorithm, $\cb$. We have
\begin{align*}
    R(T) \leq C_1 \sqrt{T} + R_{\mathcal{A}_{CB}(\mathcal{F}_{d^*})} (T-2\sqrt{T}),
\end{align*}
which proves the theorem. In the special case, where $\cb$= FALCON, the regret is \cite{falcon}
\begin{align*}
    R(T) &\leq C_1 \sqrt{T} \\
    & + \mathcal{O}\left( \sqrt{K(T-2\lfloor\sqrt{T}\rfloor)\log(|\mathcal{F}_{d^*}|(T-2\lfloor\sqrt{T}\rfloor)/\delta} \right) \\
    & \leq C_1 \sqrt{T} + \mathcal{O}\left( \sqrt{KT\log(|\mathcal{F}_{d^*}|T/\delta} \right),
\end{align*}
with probability exceeding $1-4M\delta - \delta$. Combining the above expressions yield the result.

\subsection{ Proof of Lemma~\ref{lem:falcon_model}}
Let us first show that $S_j^m$ concentrates around its expectation. We show it via a simple application of the Hoeffdings inequality.

Fix a particular $m$ and $j \in [M]$. Note that $\hat{f}_j^m$ is computed based on $2^{m-2}$ samples. Also, in the testing phase, we use a fresh set of $2^{m-2}$ samples, and so $\hat{f}_j^m$ is independent of the second set of samples, used in constructing $S_j^m$. Note that since we have $r_t(.) \in [0,1]$, we may restrict the offline regression oracle to search over functions having range $[0,1]$. This implies that, we have $\hat{f}_j^m(.) \in [0,1]$. Note that this restricted search assumption is justified since our goal is obtain an estimate of the reward function via regression function, and this assumption also features in \cite{falcon}. So the random variable $(\hat{f}_j^m(x_t,a_t) -r_t(a_t))^2$ is upper-bounded by $4$, and hence sub-Gaussian with a constant parameter. 

Note that we are using only the samples from the previous epoch. Note that in \texttt{ACB}, the regression estimate actually remains fixed over an entire epoch. Hence, conditioning on the filtration consisting of (context, action, reward) triplet upto the end of the $m-2$-th epoch, the random variables $\{(\hat{f}_j(x_t,a_t) -r_t(a_t))^2\}_{t=\tau_{m-1}/2 +1}^{\tau_{m-1}}$ (a total of $2^{m-2}$ samples) are independent. Note that similar argument is given in \cite[Section 3.1]{falcon} (the FALCON$+$ algorithm) to argue the independence of the (context, action, reward) triplet, accumulated over just the previous epoch.

Hence using Hoeffdings inequality for sub-Gaussian random variables, we have
\begin{align*}
    \Prob \left( |S_j - \mathbb{E} S_j | \geq \ell \right) \leq 2\exp (- n \ell^2/32).
\end{align*}

Note that, the conditional variance of $r_t(.)$ is finite, i.e., given $x_t = x \in \mathcal{X}$,
$\mathbb{E}[r_t(a) - f^*_{d^*}(x,a)]^2 \leq 1$, for all $a \in \mathcal{A}$. Recall that $\mathbb{E}[r_t(a) - f^*_{d^*}(x,a)]^2 =\sigma^2$, and that $\sigma$ depends on epoch $m$. We omit the subscript for notational simplicity. With this new notation, let us first look at the expression $\mathbb{E} S_j$.

\paragraph{Realizable classes}
Fix $m$ and consider $j\in [M]$ such that  $j \geq d^*$. So, for this realizable setting, we obtain the excess risk as:
\begin{align*}
    & \mathbb{E}_{x,r,a}[\hat{f}_j^m(x,a) - r(a)]^2 - \inf_{f \in \mathcal{F}_j} \mathbb{E}_{x,r,a}[f(x,a) -r(a)]^2 \\
    & = \mathbb{E}_{x,r,a}[\hat{f}_j^m(x,a) - r(a)]^2 - \mathbb{E}_{x,r,a} [f^*_{d^*} (x,a) - r(a)]^2 \\
    & = \mathbb{E}_{x,a} [\hat{f}_j^m(x,a) - f^*_{d^*}(x,a)]^2.
\end{align*}
So, we have, for the realizable function class,
\begin{align*}
   & \mathbb{E}S_j^m = \frac{1}{2^{m-2}} \mathbb{E}_{x_t,r_t,a_t}\sum_{t=1}^{2^{m-2}}[\hat{f}_j^m(x_t,a_t) - r_t(a_t)]^2 \\ 
   & = \frac{1}{2^{m-2}}\sum_{t=1}^{2^{m-2}}\mathbb{E}_{x_t,r_t,a_t} [f^*_{d^*} (x_t,a_t) - r_t(a_t)]^2 \\
   &+ \frac{1}{2^{m-2}} \sum_{t=1}^{2^{m-2}}\mathbb{E}_{x_t,a_t} [\hat{f}_j^m(x,a) - f^*_{d^*}(x,a)]^2 \\
    & \leq \sigma^2 + C_1 \log (2^{m/2}|\mathcal{F}_j|)/(2^{m-2}),
\end{align*}
Here, the first term comes from the second moment bound of $\sigma^2$, and the second term comes by setting the high probability slack as $ 2^{-m/2}$ into  \cite[Lemma 4.1]{agarwal2012contextual}\footnote{Note that for model selection, we only require this concentration result which uses a form of Freedman's inequality. In particular, we do not require the inverse gap weighting (IGW) randomization of FALCON. For any contextual bandit algorithm, that estimates the prediction function over multiple epochs, \texttt{ACB} can be employed for model selection. }. So, by applying Hoeffding's inequality, we finally have (using the bound $\mathbb{E}S^m_j \geq \sigma^2$): 
\begin{align*}
   \sigma^2  - C_3 \frac{ \sqrt{\log(1/\delta)}}{2^{m/2}} & - C_4 \frac{\sqrt{m}}{2^{m/2}}  \leq S_j^m \leq \sigma^2 + C_1 \frac{ \log (|\mathcal{F}_j|)}{2^m}+ \\
   & C_2 \frac{m}{2^m} + C_3 \frac{ \sqrt{\log(1/\delta)}}{2^{m/2}} + C_4 \frac{\sqrt{m}}{2^{m/2}}
\end{align*}
with probability at least $1-\delta/2^m$. Since we have doubling epoch, we have
\begin{align*}
    \sum_{m=1}^N 2^m \leq T,
\end{align*}
where $N$ is the number of epochs and $T$ is the time horizon. From above, we obtain $N = \mathcal{O}(\log_2 T)$. Using the bound, $m \leq N$, note that, provided
    {\color{black}
\begin{align}
\label{eq:c_zero}
2^m &\gtrsim \max\{ \log T, \log (|\mathcal{F}_{M}|), \log (1/\delta)\} ,
\end{align}}
 we have for some absolute global constant $c_0$, for any $ j \geq d^*$,
\begin{align}
  \sigma^2 - \frac{c_0}{2^{m/2}} &\leq S_{j}^m \leq  \sigma^2 + \frac{c_0}{2^{m/2}} 
  \label{eqn:lem1_two_sided_realizable}
\end{align}
with probability at least $1-\delta/2^m$. 

\paragraph{Non-Realizable classes:}For the non realizable classes, we have the following calculation. For any $f \in \mathcal{F}_j$, where $j < d^*$, we have
\begin{align*}
   & \mathbb{E}_{x,r,a}[f(x,a)-r(a)]^2 - \mathbb{E}_{x,r,a}[r(a) - f^*_{d^*}(x,a)]^2 \\
   & = \E_{x,a,r}[(f(x,a) - f^*_{d^*}(x,a))(f(x,a) + f^*_{d^*}(x,a) - 2r(a)] \\
   & = \E_{x,a}\E_{r|x} [(f(x,a) - f^*_{d^*}(x,a))(f(x,a) + f^*_{d^*}(x,a) - 2r(a)] \\
   & = \E_{x,a} [(f(x,a) - f^*_{d^*}(x,a))(f(x,a) + f^*_{d^*}(x,a) - 2 \E_{r|x} r(a)] \\
   & = \E_{x,a}[f(x,a)- f^*_{d^*}(x,a)]^2,
\end{align*}
where the third inequality follows from the fact that given context $x$, the distribution of $r$ in independent of $a$ (see \cite[Lemma 4.1]{agarwal2012contextual}).

So, we have
\begin{align*}
    \mathbb{E}_{x,r,a}[f(x,a)-r(a)]^2 & \geq \mathbb{E}_{x,r,a}[r(a) - f^*_{d^*}(x,a)]^2 \\
    & + \E_{x,a}[f(x,a)- f^*_{d^*}(x,a)]^2 \\
    & \geq \Delta + \sigma^2,
\end{align*}
where the last inequality comes from the separability assumption along with the assumption on the second moment. Since the regressor $\hat{f}_j^m \in \mathcal{F}_j$, we have
\begin{align*}
    \mathbb{E}_{x,r,a}[\hat{f}_j^m(x,a)-r(a)]^2 & \geq \mathbb{E}_{x,r,a}[r(a) - f^*_{d^*}(x,a)]^2 \\
    &+ \E_{x,a}[f(x,a)- f^*_{d^*}(x,a)]^2 \\
    & \geq \Delta + \sigma^2.
\end{align*}
Now, using $2^{m-2}$ samples, we obtain from Hoeffding's inequality that
\begin{align*}
    S_j^m \geq \Delta + \sigma^2 - C_5 \frac{ \sqrt{\log(1/\delta)}}{2^{m/2}} - C_6 \frac{\sqrt{m}}{2^{m/2}}
\end{align*}
with probability at least $1-\delta/2^m$. In particular, since {\color{black}$    2^m \gtrsim \max\{\log T,  \log (|\mathcal{F}_{M}|), \log (1/\delta)\}
$}, there is a global constant $c_1$ such that, for any $j < d^*$,
\begin{align}
    S_j^m \geq \Delta + \sigma^2 - \frac{c_1}{2^{m/2}},
    \label{eqn:lem1_non_realizable}
\end{align}
holds with probability at least $1-\delta/2^m$. 

In every phase $m$, denote by the threshold $\gamma_m := S^m_M + \frac{\sqrt{m}}{2^{m/2}}$, i.e., the Model Selection parameter in Line $8$ of Algorithm \ref{algo:falcon}.   Now, let $m_0$ be the smallest value of $m$ satisfying {\color{black}$2^m \gtrsim \max\{\frac{\log T}{\Delta^2}, \log (|\mathcal{F}_{M}|), \log (1/\delta)\} $}. We have from  Equations (\ref{eqn:lem1_two_sided_realizable}) and (\ref{eqn:lem1_non_realizable}) and a union bound over the $M$ classes that, with probability at-least $1- \sum_{m\geq 1}2M\delta 2^{-m}$, for all phases $m \geq m_0$, 
\begin{align*}
       S_j^m &\geq \sigma^2 + \Delta - \frac{c_1}{2^{m/2}}, \text{ for all } 1 \leq j < d^{*}, \\
    \sigma^2 - \frac{c_0}{2^{m/2}}\leq S_j^m &\leq \sigma^2 + \frac{c_0}{2^{m/2}} , \text{ for all } j \geq  d^{*}.
\end{align*}
The preceding display, along with the fact that the threshold $\gamma_m = S^M_m + \sqrt{\frac{m}{2^m}}$, gives that, with probability at-least $1-2M\delta$ and all phases $m \geq m_0$,
\begin{align*}
      S^{m}_{d^*} \leq \sigma^2 + \frac{c_0}{2^{m/2}}  \leq  \sigma^2 - \frac{c_0}{2^{m/2}} + \frac{\sqrt{m}}{2^{m/2}}   \leq \gamma_m \leq \sigma^2 + \frac{c_0}{2^{m/2}} \\
      + \frac{\sqrt{m}}{2^{m/2}}
        \leq \sigma^2 + \Delta - \frac{c_1}{2^{m/2}}.
\end{align*}
The second inequality follows since $2^m \gtrsim \frac{\log T}{\Delta^2} $, by definition of $m_0$. The above equations guarantee that, with probability at-least $1-2M\delta$, in all phases $m \geq m_0$, the model selection procedure in Line $8$ of Algorithm \ref{algo:falcon}, identifies the correct class $d^{*}$.



\subsection{Proof of Theorem~\ref{thm:falcon}}
The above calculation shows that as soon as
\begin{align*}
   2^m \gtrsim \max\{ \log (|\mathcal{F}_M|), \log (1/\delta), \log T \Delta^{-2} \},
\end{align*}
the model selection procedure will succeed with high probability. Until the above condition is satisfied, we do not have any handle on the regret and hence the regret in that phase will be linear. This corresponds the first term in the regret expression. Suppose $m^*$ be the epoch index where the conditions of Lemma~\ref{lem:falcon_model} hold. Lemma \ref{lem:falcon_model} gives that the total number of rounds till the beginning of phase $m^{*}$ is upper bounded by $\mathcal{O}( \max\{ \log (|\mathcal{F}_M|), \log (1/\delta), \log T \Delta^{-2} \} )$, where $\mathcal{O}$ hides global absolute constants. Then, the total regret is given by
\begin{align*}
    R(T) & \leq  \mathcal{O}( \max\{ \log (|\mathcal{F}_M|), \log (1/\delta), \log T \Delta^{-2} \} ) \\
    & + \sum_{m=m^*}^N R_{\cb(\mathcal{F}_{d^*})}(m-th \text{ epoch})
\end{align*}
with probability exceeding $1-2M\delta$, where $N$ is the number of epochs. We have
\begin{align*}
    R(T) &\leq  \mathcal{O}( \max\{ \log (|\mathcal{F}_M|), \log (1/\delta), \log T \Delta^{-2} \} ) \\
    & + \sum_{m=m^*}^N R_{\cb(\mathcal{F}_{d^*})}(m-th \text{ epoch}) \\
    &\leq  \mathcal{O}( \max\{ \log (|\mathcal{F}_M|), \log (1/\delta), \log T \Delta^{-2} \} ) \\
    & + \sum_{m=1}^N R_{\cb(\mathcal{F}_{d^*})}(m-th \text{ epoch}) \\
    & \leq \mathcal{O}( \max\{ \log (|\mathcal{F}_M|), \log (1/\delta), \log T \Delta^{-2} \} ) \\
    &+ R_{\cb(\mathcal{F}_{d^*})}(T),
    \end{align*}
which proves the theorem.

Now, let us focus on the case where FALCON is used as the base algorithm. For the $m-th$ epoch, with $m \geq m^*$, the regret is given by we have (see \cite{falcon}):
\begin{align*}
    \sum_{m=m^*}^N \mathcal{O}\left( \sqrt{K(\tau_m-\tau_{m-1})\log(|\mathcal{F}_{d^*}|(\tau_m-\tau_{m-1})/\delta_m} \right)
\end{align*}
with probability at least $1-\delta_m$. So, the total regret is given by
\begin{multline*}
    R(T) \leq \mathcal{O}( \max\{ \log (|\mathcal{F}_M|), \log (1/\delta), \log T \Delta^{-2})\} \\ + \sum_{m=m^*}^N \mathcal{O}\left( \sqrt{K(\tau_m-\tau_{m-1})\log(|\mathcal{F}_{d^*}|(\tau_m-\tau_{m-1})/\delta_m} \right),
    \end{multline*}
    with probability at-least $1-\delta - 2M\delta$.
    Simplifying the summation, we get
    \begin{align*}
    & \sum_{m=m^*}^N \mathcal{O}\left( \sqrt{K(\tau_m-\tau_{m-1})\log(|\mathcal{F}_{d^*}|(\tau_m-\tau_{m-1})/\delta_m} \right) \\
    & \leq  \sum_{m=1}^N \mathcal{O}\left( \sqrt{K(\tau_m-\tau_{m-1})\log(|\mathcal{F}_{d^*}|(\tau_m-\tau_{m-1})/\delta_m} \right)\\
     & \leq   \mathcal{
     O}(\sqrt{K \log(|\mathcal{F}_{d^*}|(T)/\delta}) \sum_{m=1}^N  \sqrt{\tau_m-\tau_{m-1}},
\end{align*}
considering the leading terms. Note that, with $\tau_m = 2^m$, the epoch length $\tau_m - \tau_{m-1}$ doubles with $m$. Let the length of the $N$-th epoch is $T_N$. We have
\begin{align*}
    \sum_{i=1}^N \sqrt{\tau_m-\tau_{m-1}} & = \sqrt{T_N}\left(1 + \frac{1}{\sqrt{2}} + \frac{1}{2} + \ldots N\text{-th term} \right)\\
    & \leq \sqrt{T_N}\left(1 + \frac{1}{\sqrt{2}} + \frac{1}{2} + ... \right) \\
    &= \frac{\sqrt{2}}{\sqrt{2} -1} \sqrt{T_N}  \leq \frac{\sqrt{2}}{\sqrt{2} -1} \sqrt{T},
    \end{align*}
and this completes the proof of the theorem.

\subsection{Proof of Theorem~\ref{thm:etc_inf}}
\paragraph{Case I: Realizable Class} Consider $j \geq d^*$. Using calculations similar to the finite cardinality setting, we obtain
\begin{align*}
    \mathbb{E}S_j \leq \sigma^2 + \xi_{\mathcal{F}_j,(1/T^{1/4})}(\sqrt{T})+ 2(1/T^{1/4}),
\end{align*}
where we use the definition of $\xi(.)$, as given in Assumption~\ref{asm:pred_error_inf}. Hence, invoking Hoeffding's inequality, we obtain
\begin{align*}
    S_j &\leq \sigma^2 + \xi_{\mathcal{F}_j,(1/T^{1/4})}(\sqrt{T}) + 2(1/T^{1/4}) \\
    &+ C_1 T^{-1/4}\sqrt{\log(1/\delta)} \\
    &\leq \sigma^2 + \xi_{\mathcal{F}_j,(1/T^{1/4})}(\sqrt{T}) + C_1 T^{-1/4}\sqrt{\log(1/\delta)}
\end{align*}
with probability at least $1-2\delta$. We also have (from 2-sided Hoeffding's)
\begin{align*}
    S_j \geq \sigma^2 - C_2 T^{-1/4}\sqrt{\log(1/\delta)}
\end{align*}

\paragraph{Case II: Non-realizable Class} We now consider the setting where $j <d^*$, meaning that $f^*_{d^*}$ does not lie in $\mathcal{F}_j$. In this case, similar to above, we have
\begin{align*}
   \mathbb{E}S_j \geq \Delta + \sigma^2,
\end{align*}
and hence
\begin{align*}
    S_j & \geq \mathbb{E}S_j - \sqrt{\frac{32 \log(1/\delta)}{ \sqrt{T}}}\\
    & \geq \Delta + \sigma^2 - \sqrt{\frac{32 \log(1/\delta)}{ \sqrt{T}}}.
\end{align*}

Now, with the threshold, $\gamma = S_M + \sqrt{\frac{\log T}{\sqrt{T}}}$, provided
\begin{align*}
    T \gtrsim (\log T) \max \left( \log \left( T^{1/4} \xi_{\mathcal{F}_M, (1/T^{1/4})} \right), \Delta^{-4}, \log(1/\delta) \right),
\end{align*}
the model selection procedure succeeds with probability at least $1-2M\delta$, where we do a calculation similar to the proof of Lemma~\ref{lem:etc}.

After obtaining the correct model class, the regret expression comes directly from \cite{falcon} in the infinite function class setting.

\subsection{Proof of Theorem~\ref{thm:falcon_inf} }
The proof follows by combining the proof of Theorem~\ref{thm:falcon} and \ref{thm:etc_inf}. 

\noindent For the realizable classes, we have (from Assumption~\ref{asm:pred_error_inf} and converting the conditional expectation to unconditional one with probability slack as $1/2^{m/2}$, similar to the proof of Lemma~\ref{lem:falcon_model}),
\begin{align*}
  \mathbb{E}S_j^m \leq \sigma^2 + \xi_{\mathcal{F}_j,1/2^{m/2}}(2^{m-2}) + 2 (\frac{1}{2^{m/2}}),
\end{align*}
and as a result
\begin{align*}
    S_j^m \leq \sigma^2 + \xi_{\mathcal{F}_j,1/2^{m/2}}(2^{m-2}) + C_1 \frac{ \sqrt{\log(1/\delta)}}{2^{m/2}} + C_2 \frac{\sqrt{m}}{2^{m/2}}
\end{align*}
with probability at least $1-2\delta/2^m$.

Similarly, for non-realizable classes we obtain
\begin{align*}
    S_j^m \geq \Delta + \sigma^2 - C_3 \frac{ \sqrt{\log(1/\delta)}}{2^{m/2}} - C_4 \frac{\sqrt{m}}{2^{m/2}}
\end{align*}
with probability at least $1-\delta/2^m$.

Now, suppose we choose the threshold $\gamma = S^m_M + \frac{\sqrt{m}}{2^{m/2}}$. Finally, we say that provided
\begin{align*}
    2^m \gtrsim (\log T) \max\{ \max_m  2^{m/2} \, \xi_{\mathcal{F}_M,1/2^{m/2}}(2^{m-2}), \\
    \log (1/\delta), \Delta^{-2} \},
\end{align*}
the model selection procedure succeeds with probability exceeding
\begin{align*}
    1- \sum_{m=1}^\infty 2 M \delta/2^m \geq 1-2M\delta.
\end{align*}
The rest of the proof follows similarly to Theorem~\ref{thm:falcon}, and we omit the details here.

\section{Model Selection for Linear Stochastic bandits}
\subsection{Proof of Theorem \ref{thm:adaptive_dimension}}

We shall need the following lemma from
\cite{linear_reg_guarantees}, on the behaviour of linear regression estimates.

\begin{lemma}
If $M \geq d$ and satisfies $M = O \left( \left(\frac{1}{\varepsilon^2} + d \right) \ln \left(\frac{1}{\delta}\right)\right)$, and $\widehat{\theta}^{(M)}$ is the least-squares estimate of $\theta^*$, using the $M$ random samples for feature, where each feature is chosen uniformly and independently on the unit sphere in $d$ dimensions, then with probability $1$, $\widehat{\theta}$ is well defined (the least squares regression has an unique solution). Furthermore,
\begin{align*}
    \mathbb{P}[||\widehat{\theta}^{(M)} - \theta^*||_{\infty} \geq \varepsilon] \leq \delta.
\end{align*}
\label{lem:reg_bounds}
\end{lemma}

We shall now apply the theorem as follows. Denote by $\widehat{\theta}_{i}$ to be the estimate of $\theta^*$ at the beginning of any phase $i$, using all the samples from random explorations in all phases less than or equal to $i-1$.

\begin{remark}
The choice $T_0 := O \left( d^2 \ln^2 \left( \frac{1}{\delta} \right) \right)$ in Equation (\ref{eqn:T_0_defn}) is chosen such that from Lemma \ref{lem:linear_reg_guarantee}, we have that 
\begin{align*}
    \mathbb{P}\left[||\widehat{\theta}^{(\lceil \sqrt{T_0}\rceil)} - \theta^*||_{\infty} \geq \frac{1}{2} \right] \leq \delta
\end{align*}
\label{remark:choice_of_TO}
\end{remark}

\begin{lemma}
Suppose $T_0 = O \left( d^2 \ln^2 \left( \frac{1}{\delta} \right) \right)$ is set according to Equation (\ref{eqn:T_0_defn}). Then, for all phases $i \geq 4$, 
\begin{align}
    \mathbb{P} \left[ || \widehat{\theta}_i - \theta^*||_{\infty} \geq 2^{-i}\right] \leq \frac{\delta}{2^i},
    \label{eqn:cordinates_guarantee}
\end{align}
where $\widehat{\theta}_i$ is the estimate of $\theta^*$ obtained by solving the least squares estimate using all random exploration samples until the beginning of phase $i$.
\label{lem:linear_reg_guarantee}
\end{lemma}

\begin{proof}
The above lemma follows directly from Lemma \ref{lem:reg_bounds}. 
Lemma \ref{lem:reg_bounds} gives that if $\widehat{\theta}_i$ is formed by solving the least squares estimate with at-least $M_i := O \left( \left( 4^i + d\right)\ln \left( \frac{2^i}{\delta} \right) \right)$ samples, then the guarantee in Equation (\ref{eqn:cordinates_guarantee}) holds. 
However, as $T_0 = O \left( (d+1) \ln \left( \frac{2}{\delta} \right) \right)$, we have naturally that $M_i \leq 4^i i \sqrt{T_0}$.
The proof is concluded if we show that at the beginning of phase $i \geq 4$, the total number of random explorations performed by the algorithm exceeds 
$i4^i \lceil \sqrt{T_0} \rceil$.
Notice that at the beginning of any phase $i \geq 4$, the total number of random explorations that have been performed is 
\begin{align*}
    \sum_{j=0}^{i-1}6^i \lceil \sqrt{T_0} \rceil & = \lceil \sqrt{T_0} \rceil \frac{6^i-1}{4},\\
    &\geq i4^i \lceil \sqrt{T_0} \rceil,
\end{align*}
where the last inequality holds for all $i \geq 10$.
\end{proof}

The following corollary follows from a straightforward union bound.
\begin{corollary}
\begin{align*}
    \mathbb{P} \left[ \bigcap_{i \geq 4} || \left\{\widehat{\theta}_i - \theta^*||_{\infty} \leq 2^{-i} \right\} \right] \geq 1 - \delta.
\end{align*}
\label{cor:good_coordinate_behaviour}
\end{corollary}
\begin{proof}
This follows from a simple union bound as follows.
\begin{align*}
        &\mathbb{P} \left[ \bigcap_{i \geq 4} \left\{||  \widehat{\theta}_i - \theta^*||_{\infty} \leq 2^{-i} \right\} \right] \\
        &= 1 - \mathbb{P} \left[ \bigcup_{i \geq 4}\left\{|| \widehat{\theta}_i - \theta^*||_{\infty} \geq 2^{-i} \right\} \right],\\
        &\geq 1 - \sum_{i \geq 4} \mathbb{P} \left[ ||\widehat{\theta}_i-\theta^*||_{\infty} \geq 2^{-i} \right],\\
        &\geq 1 - \sum_{i\geq 4} \frac{\delta}{2^i},\\
        &\geq 1 - \sum_{i\geq 2}\frac{\delta}{2^i},\\
        &=1-\frac{\delta}{2}.
\end{align*}
\end{proof}

We are now ready to conclude the proof of Theorem \ref{thm:adaptive_dimension}.

\begin{proof}[Proof of Theorem \ref{thm:adaptive_dimension}]

We know from Corollary \ref{cor:good_coordinate_behaviour}, that with probability at-least $1 - \delta$, for all phases $i \geq 10$, we have $||\widehat{\theta}_i-\theta^*||_{\infty} \leq 2^{-i}$.
Call this event $\mathcal{E}$. 
Now, consider the phase $i(\gamma):= \max \left(10,\log_{2} \left( \frac{1}{\gamma} \right)\right)$. 
Now, when event $\mathcal{E}$ holds, then for all phases $i \geq i(\gamma)$, $\mathcal{D}_i$ is the correct set of $d^*$ non-zero coordinates of $\theta^*$. Thus, with probability at-least $1-\delta$, the total regret upto time $T$ can be upper bounded as follows
\begin{align}
    R_T  & \leq \sum_{j=0}^{i(\gamma)-1} \left( 36^i T_0 + 6^i \lceil \sqrt{T_0} \rceil \right) \notag \\
    & + \sum_{j \geq i(\gamma)}^{\bigg\lceil \log_{36} \left( \frac{T}{T_0} \right) \bigg\rceil} \text{Regret}(\text{OFUL}(1,\delta_i;36^i T_0) \nonumber \\
    & + \sum_{j=i(\gamma)}^{\bigg\lceil \log_{36} \left( \frac{T}{T_0} \right) \bigg\rceil} 6^j \lceil \sqrt{T_0} \rceil.
    \label{eqn:regret_decomposition_dimension}
\end{align}
The term $\text{Regret(OFUL}(L,\delta,T)$ denotes the regret of the OFUL algorithm \cite{oful}, when run with parameters $L \in \mathbb{R}_+$, such that $\|\theta^*\| \leq L$, and $\delta \in (0,1)$ denotes the probability slack and $T$ is the time horizon.
Equation (\ref{eqn:regret_decomposition_dimension}) follows, since the total number of phases is at-most $\bigg\lceil \log_{36} \left( \frac{T}{T_0} \right) \bigg\rceil$.
Standard result from \cite{oful} give us that, with probability at-least $1-\delta$, we have
\begin{align*}
  \text{Regret}(\text{OFUL}(1,\delta; T) \leq 
  4 \sqrt{Td^* \ln \left( 1 + \frac{T}{d^*} \right)} \\
  \times \left(1 + \sigma \sqrt{2 \ln \left( \frac{1}{\delta}\right)+d^*\ln \left( 1 + \frac{T}{d} \right)} \right).
\end{align*}
Thus, we know that with probability at-least $1 - \sum_{i \geq 4}\delta_i \geq 1-\frac{\delta}{2}$, for all phases $i \geq i(\gamma)$, the regret in the exploration phase satisfies
\begin{align}
    & \text{Regret}(\text{OFUL}(1,\delta_i;36^i T_0)  \leq  4\sqrt{d^*36^iT_0 \ln \left(1+ \frac{36^iT_0}{d^*}\right)} \nonumber \\
    & \times \left(1 + \sigma\sqrt{2 \ln \left( \frac{2^i}{\delta}\right) + d^* \ln \left( 1+\frac{36^iT_0}{d^*}\right)} \right).
    \label{eqn:intermediate1}
\end{align}

In particular, for all phases $i \in [i(\gamma), \lceil \log_{36}\left( \frac{T}{T_0} \right)]$, with probability  at-least $1-\frac{\delta}{2}$, we have 
\begin{align}
     & \text{Regret}(\text{OFUL}(1,\delta_i;36^i T_0) \leq  4\sqrt{d^*36^iT_0 \ln \left(1+ \frac{T}{d^*}\right)} \nonumber \\
     & \times \left(1 + \sigma \sqrt{2 \ln \left( \frac{T}{T_0\delta} \right) + d^* \ln \left( 1+\frac{T}{d^*}\right)} \right), \nonumber\\
     &= \mathcal{C}(T,\delta,d^*) \sqrt{36^iT_0},
     \label{eqn:intermediate_eqn1}
\end{align}
where the constant captures all the terms that only depend on $T$, $\delta$ and $d^*$. We can write that constant as 
\begin{align*}
   \mathcal{C}(T,\delta,d^*) = 4 \sqrt{d^*\ln \left( 1 + \frac{T}{d^*} \right)} \\
   \times \left(1 + \sigma \sqrt{2 \ln \left( \frac{T}{T_0\delta} \right) + d^* \ln \left( 1+\frac{T}{d^*}\right)} \right).
\end{align*}

Equation (\ref{eqn:intermediate_eqn1}) follows, by substituting $i \leq \log_{36} \left( \frac{T}{T_0}\right)$ in all terms except the first $36^i$ term in Equation (\ref{eqn:intermediate1}).
As Equations (\ref{eqn:intermediate_eqn1}) and (\ref{eqn:regret_decomposition_dimension}) each hold with probability at-least $1-\frac{\delta}{2}$, we can combine them to get that with probability at-least $1-\delta$,
\begin{align*}
    R_T &\leq 2T_036^{i(\gamma)} + \sum_{j=0}^{\log_{36}\left( \frac{T}{T_0} \right)+1} \mathcal{C}(T,\delta,d^*) \sqrt{36^jT_0} \\
    & +  \lceil\sqrt{T_0} \rceil 6^{\log_{36} \left( \frac{T}{T_0}\right)} ,\\
    &\leq \mathcal {O} \left(T_036^{i(\gamma)} +  \sqrt{T} + \mathcal{C}(T,\delta,d^*)\sum_{j=0}^{\log_{36}\left( \frac{T}{T_0} \right)+1}\sqrt{36^j T_0} \right),\\
    &\stackrel{(a)}{\leq} \mathcal{O} \left( T_0 \frac{2}{{\gamma^{5.18}}} + \sqrt{T} + \sqrt{T}\mathcal{C}(T,\delta,d^*) \right),\\
    &= \mathcal{O} \left( \frac{d^2}{{\gamma^{5.18}}} \ln^2 \left( \frac{1}{\delta} \right) \right) + \widetilde{O} \left( d^*\sqrt{ T \ln\left( \frac{1}{\delta} \right)} \right).
\end{align*}
Step $(a)$ follows from $36 \leq 2^{5.18}$.

\end{proof}

\section{{\ttfamily ALB-Dim} for Stochastic Contextual Bandits with Finite Arms}
\label{appendix-comparision}

\subsection{ALB-Dim Algorithm for the Finite Armed Case}

The algorithm given in Algorithm \ref{algo:main_algo_dimensions_foster}  is identical to the earlier Algorithm \ref{algo:main_algo_dimensions_unknown}, except in Line $8$, this algorithm uses {\ttfamily SupLinRel} of \cite{chu2011contextual} as opposed to OFUL used in the previous algorithm.
In practice, one could also use {\ttfamily LinUCB} of \cite{chu2011contextual} in place of {\ttfamily SupLinRel}. However, we choose to present the theoretical argument using {\ttfamily SupLinRel}, as unlike {\ttfamily LinUCB}, has an explicit closed form regret bound (see \cite{chu2011contextual}).
The pseudocode is provided in Algorithm \ref{algo:main_algo_dimensions_foster}.

In phase $i \in \mathbb{N}$, the {\ttfamily SupLinRel} algorithm is instantiated with input parameter $36^iT_0$ denoting the time horizon, slack parameter $\delta_i \in (0,1)$, dimension $d_{\mathcal{M}_i}$ and feature scaling $b(\delta)$. We explain the role of these input parameters. The dimension ensures that {\ttfamily SupLinRel} plays from the restricted dimension $d_{\mathcal{M}_i}$. The feature scaling implies that when a context $x \in \mathcal{X}$ is presented to the algorithm, the set of $K$ feature vectors, each of which is $d_{\mathcal{M}_i}$ dimensional are $\frac{\phi^{d_{\mathcal{M}_i}}(x,1)}{b(\delta)},\cdots, \frac{\phi^{d_{\mathcal{M}_i}}(x,K)}{b(\delta)}$. The constant $b(\delta) \coloneqq O \left( \tau \sqrt{\log \left( \frac{TK}{\delta} \right)} \right)$ is chosen such that 
\begin{align*}
    \mathbb{P}\left[\sup_{t \in [0,T],a \in \mathcal{A}} \| \phi^M(x_t,a) \|_2 \geq b(\delta) \right] \leq \frac{\delta}{4}.
\end{align*}
Such a constant exists since $(x_t)_{t \in [0,T]}$ are i.i.d. and $\phi^M(x,a)$ is a sub-gaussian random variable with parameter $4\tau^2$, for all $a \in \mathcal{A}$. Similar idea was used in \cite{foster_model_selection}.

\begin{algorithm}[t!]
  \caption{Adaptive Linear Bandit (Dimension) with Finitely Many arms}
  \begin{algorithmic}[1]
 \STATE  \textbf{Input:} Initial Phase length $T_0$ and slack $\delta > 0$.
 \STATE $\widehat{\beta}_0 = \mathbf{1}$, $T_{-1}=0$
 \FOR {Each epoch $i \in \{0,1,2,\cdots\}$}
 \STATE $T_i = 36^{i} T_0$, $\quad$  $\varepsilon_i \gets \frac{1}{2^{i}}$, $\quad$  $\delta_i \gets \frac{\delta}{2^{i}}$
 \STATE $\mathcal{D}_i := \{i : |\widehat{\beta}_i| \geq \frac{\varepsilon_i}{2} \}$
 \STATE $\mathcal{M}_i \coloneqq \inf \{ m : d_m \geq \max \mathcal{D}_i \}$.
 \FOR {Times $t \in \{T_{i-1}+1,\cdots,T_i\}$}
 \STATE Play according to SupLinRel of \cite{linRel} with time horizon of $36^i T_0$ with parameters $\delta_i \in (0,1)$, dimension $d_{\mathcal{M}_i}$ and feature scaling $b(\delta) \coloneqq  O \left( \tau \sqrt{\log \left( \frac{TK}{\delta} \right)} \right)$.
 \ENDFOR
 \FOR {Times $t \in \{T_i+1,\cdots,T_i + 6^i\sqrt{T_0}\}$}
 \STATE Play an arm from the action set ${\mathcal{A}}$ chosen uniformly and independently at random.
 \ENDFOR
 \STATE $\boldsymbol{\alpha}_i \in \real^{S_i \times d}$ with each row being  the arm played during all random explorations in the past.
 \STATE $\boldsymbol{y}_i \in \real^{S_i}$  with $i$-th entry being the observed reward at the $i$-th random exploration in the past
 \STATE $\widehat{\beta}_{i+1} \gets (\boldsymbol{\alpha}_i^T\boldsymbol{\alpha}_i)^{-1}\boldsymbol{\alpha}_i\mathbf{y}_i$, is a $d$ dimensional vector
 \ENDFOR
  \end{algorithmic}
  \label{algo:main_algo_dimensions_foster}
\end{algorithm}

\subsection{Regret Guarantee for Algorithm \ref{algo:main_algo_dimensions_foster}}

In order to specify a regret guarantee, we will need to specify the value of $T_0$. 
We do so as before. 
For any $N$, denote by $\lambda_{max}^{(N)}$ and $\lambda_{min}^{(N)}$ to be the maximum and minimum eigen values of the following matrix: $\boldsymbol{\Sigma}^N := \mathbb{E} \left[ \frac{1}{K} \sum_{j=1}^K \sum_{t=1}^N \phi^M(x_t,j)\phi^M(x_t,j)^T \right]$, where the expectation is with respect to $(x_t)_{t \in [T]}$ which is an i.i.d. sequence with distribution $\mathcal{D}$. First, given the distribution of $x \sim \mathcal{D}$, one can (in principle) compute $\lambda_{max}^{(N)}$ and $\lambda_{min}^{(N)}$ for any $N \geq 1$. Furthermore, from the assumption on $\mathcal{D}$, $\lambda_{min}^{(N)} = \widetilde{O} \left( \frac{1}{\sqrt{d}} \right) > 0$ for all $N \geq 1$.
Choose $T_0 \in \mathbb{N}$ to be the smallest integer such that
\begin{align}
    \sqrt{T_0} \geq b(\delta)\max \bigg ( \frac{32\sigma^2}{(\lambda_{min}^{(\lceil \sqrt{T_0} \rceil)})^2}\ln (2d/\delta), \notag \\
    \frac{4}{3} \frac{(6\lambda_{max}^{(\lceil \sqrt{T_0}\rceil)}+\lambda_{min}^{(\lceil \sqrt{T_0}\rceil)})(d+\lambda_{max}^{(\lceil \sqrt{T_0}\rceil)})}{(\lambda_{min}^{(\lceil \sqrt{T_0}\rceil)})^2}\ln ( 2d/\delta) \bigg ).
    \label{eqn:T_0_defn_foster}
\end{align}
As before, it is easy to see that 
\begin{align*}
    T_0 = O \left( d^2 \ln^2 \left( \frac{1}{\delta}\right) \tau^2 \ln \left( \frac{TK}{\delta}\right) \right).
\end{align*}
Furthermore, following the same reasoning as in Lemmas \ref{lem:linear_reg_guarantee} and \ref{lem:reg_bounds}, one can verify that for all $i \geq 4$, $\mathbb{P} \left[ \| \widehat{\beta}_{i-1} - \beta^* \|_{\infty} \geq 2^{-i} \right] \leq \frac{\delta}{2^i} $.

\begin{theorem}
Suppose Algorithm \ref{algo:main_algo_dimensions_foster} is run with input parameters $\delta \in (0,1)$, and $T_0$ as given in Equation (\ref{eqn:T_0_defn_foster}), then with probability at-least $1-\delta$, the regret after a total of $T$ arm-pulls satisfies
\begin{align*}
R_T \leq C T_0  \frac{1}{{\gamma^{5.18}}}  +  (1 + \ln(2K T \ln T))^{3/2} \sqrt{T d_{m^*}} + \sqrt{T}.
\end{align*}
The parameter $\gamma > 0$ is the minimum  magnitude of the non-zero coordinate of $\beta^*$, i.e., $\gamma = \min \{|\beta^*_i| : \beta^*_i \neq 0 \}$.
\label{thm:adaptive_dimension_foster}
\end{theorem}

In order to parse the above theorem, the following corollary is presented.

\begin{corollary}
Suppose Algorithm \ref{algo:main_algo_dimensions_foster} is run with input parameters $\delta \in (0,1)$, and $T_0 = \widetilde{O} \left(d^2\ln^2 \left( \frac{1}{\delta} \right) \right)$ given in Equation (\ref{eqn:T_0_defn_foster}) , then with probability at-least $1-\delta$, the regret after $T$ times satisfies
\begin{align*}
        R_T &\leq O \left( \frac{d^2}{{\gamma^{5.18}}} \ln^2 ( d/\delta) \tau^2 \ln \left( \frac{TK}{\delta}\right) \right)  + \widetilde{O} (  \sqrt{ T d^*_m}).
\end{align*}
\end{corollary}

\begin{proof}[Proof of Theorem \ref{thm:adaptive_dimension_foster}]

The proof proceeds identical to that of Theorem \ref{thm:adaptive_dimension}. Observe from Lemmas \ref{lem:reg_bounds} and \ref{lem:linear_reg_guarantee}, that the choice of $T_0$ is such that for all phases $i \geq 1$, the estimate $\mathbb{P}\left[ \| \widehat{\beta}_{i-1} - \beta^* \|_{\infty} \geq 2^{-i} \right] \leq \frac{\delta}{2^i}$.
Thus, from an union bound, we can conclude that 
\begin{align*}
    \mathbb{P} \left[ \cup_{i \geq 4} \| \widehat{\beta}_{i-1} - \beta^* \|_{\infty} \geq 2^{-i}  \right] \leq \frac{\delta}{4}.
\end{align*}
Thus at this stage,  with probability at-least $1-\frac{\delta}{2}$, the following events holds.
\begin{itemize}
    \item $\sup_{t \in [0,T],a \in \mathcal{A}} \| \phi^M(x_t,a) \|_2 \leq b(\delta)$
    \item $ \| \widehat{\beta}_{i-1} - \beta^* \|_{\infty} \leq 2^{-i}$, for all $i \geq 10$.
\end{itemize}
Call these events as $\mathcal{E}$.
As before, let $\gamma > 0$ be the smallest value of the non-zero coordinate of $\beta^*$. Denote by the phase $i(\gamma) \coloneqq \max \left( 10, \log_{2} \left( \frac{2}{\gamma} \right)\right)$. Thus, under the event $\mathcal{E}$, for all phases $i \geq i(\gamma)$, the dimension $d_{\mathcal{M}_i} = d_m^* $, i.e., the SupLinRel is run with the correct set of dimensions.  

It thus remains to bound the error by summing over the phases, which is done identical to that in Theorem \ref{thm:adaptive_dimension}.
With probability, at-least $1-\frac{\delta}{2} - \sum_{i \geq 4}\delta_i \geq 1-\delta $,
\begin{align*}
R_T & \leq \sum_{j=0}^{i(\gamma)-1} \left( 36^jT_0 + 6^j \sqrt{T_0} \right) \\
&+ \sum_{j=i(\gamma)}^{\bigg\lceil \log_{36} \left( \frac{T}{T_0} \right) \bigg\rceil} \text{Regret(SupLinRel)}(36^iT_0, \delta_i,d_{\mathcal{M}_i,b(\delta)}) \\
& + \sum_{j=i(\gamma)}^{\bigg\lceil \log_{36} \left( \frac{T}{T_0} \right) \bigg\rceil} 6^j\sqrt{T_0},
\end{align*}
where $\text{Regret(SupLinRel)}(36^iT_0, \delta_i,d_{\mathcal{M}_i,b(\delta)}) \leq  C (1 + \ln(2K 36^iT_0 \ln 36^iT_0))^{3/2}\sqrt{36^iT_0 d_{\mathcal{M}_i}} + 2\sqrt{36^iT_0}$. This expression follows from Theorem $6$ in \cite{linRel}. We now use this to bound each of the three terms in the display above.
Notice from straightforward calculations that the first term is bounded by $2T_036^{i(\gamma)}$ and the last term is bounded above by $36 \lceil \sqrt{T_0} \rceil 6^{\log_{36} \left( \frac{T}{T_0} \right)}$
respectively. We now bound the middle term as
\begin{align*}
&\sum_{j=i(\gamma)}^{\bigg\lceil \log_{36} \left( \frac{T}{T_0} \right) \bigg\rceil} \text{Reg(SupLinRel)}(36^jT_0, \delta_i,d_{m}^*,b(\delta))  \\
&\leq b(\delta) \bigg( \sum_{j=i(\gamma)}^{\bigg\lceil \log_{36} \left( \frac{T}{T_0} \right) \bigg\rceil} C (1 + \ln(2K 36^iT_0 \ln 36^iT_0))^{3/2} \\
& \quad  \sqrt{36^iT_0 d_{\mathcal{M}_i}} + 2\sqrt{36^iT_0} \bigg).
 \end{align*}
 The first summation can be bounded as 
 \begin{align*}
   & \sum_{j=i(\gamma)}^{\bigg\lceil \log_{36} \left( \frac{T}{T_0} \right) \bigg\rceil} C (1 + \ln(2K 36^iT_0 \ln 36^iT_0))^{3/2}\sqrt{36^iT_0 d_{\mathcal{M}_i}} \\
   &\leq \sum_{j=i(\gamma)}^{\bigg\lceil \log_{36} \left( \frac{T}{T_0} \right) \bigg\rceil} C (1 + \ln(2K T \ln T))^{3/2}\sqrt{36^iT_0 d_{m}^*},\\
   &= C_1 (1 + \ln(2K T \ln T))^{3/2} \sqrt{T d_m^*},
 \end{align*}
 and the second by 
 \begin{align*}
    \sum_{j=i(\gamma)}^{\bigg\lceil \log_{36} \left( \frac{T}{T_0} \right) \bigg\rceil} 2\sqrt{36^iT_0} \leq C_1  \sqrt{T}.
 \end{align*}
 
 Thus, with probability at-least $1-\delta$, the regret of Algorithm \ref{algo:main_algo_dimensions_foster} satisfies
 \begin{align*}
R_T \leq 2T_0 36^{i(\gamma)} +C (1 + \ln(2K T \ln T))^{3/2} \sqrt{T d_m^*} + C_2 \sqrt{T}, 
 \end{align*}
 where $i(\gamma) \coloneqq \max \left( 10, \log_{2} \left( \frac{2}{\gamma} \right) \right)$. Thus,

     \begin{align*}
R_T \leq C T_0  \frac{2}{{\gamma^{5.18}}}  + C (1 + \ln(2K T \ln T))^{3/2} \sqrt{T d_m^*} + C_1 \sqrt{T},
 \end{align*} 
as $36 \leq 2^{5.18}$

\end{proof}
 \section{Numerical Experiments}
\label{sec:experiments}

\begin{figure*}[h!]
    \centering
    \subfigure[$d^*=20$, $d=500$ ]{{\includegraphics[height = 4cm,width=4cm]{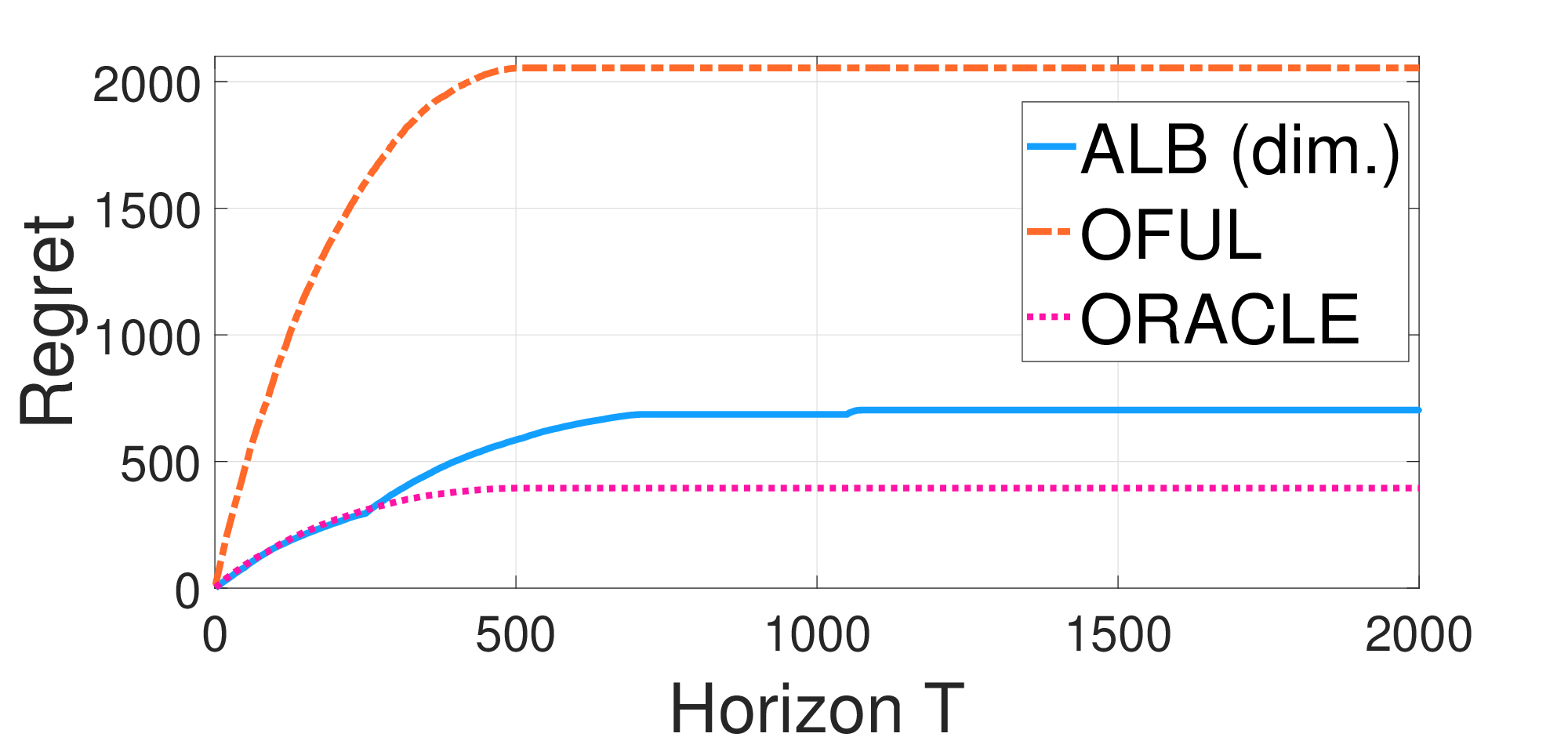} }}
    \subfigure[$d^*=20$, $d=200$ ]{{\includegraphics[height = 4cm,width=4cm]{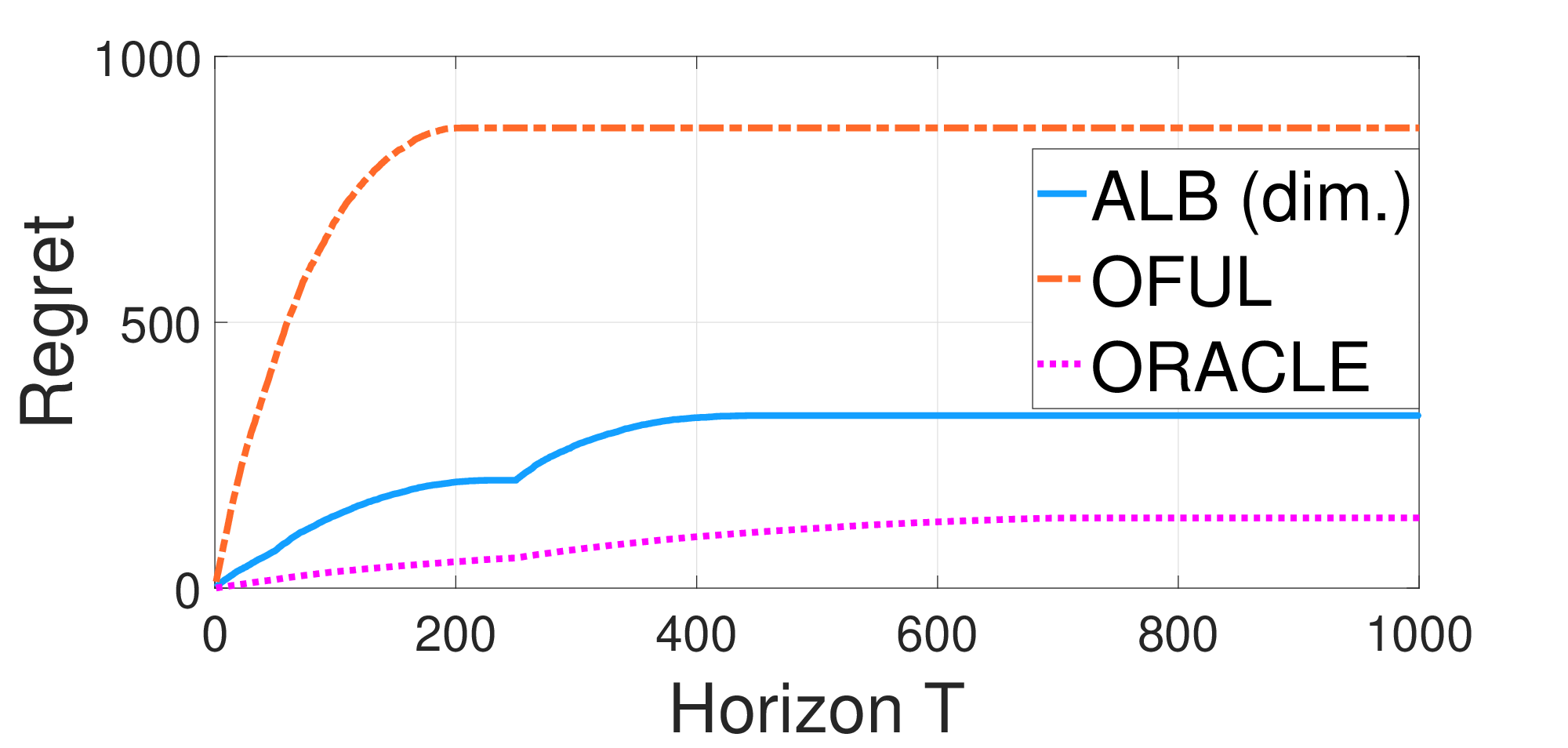} }}
    \subfigure[Dimension refinement ]{{\includegraphics[height = 4cm,width=4cm]{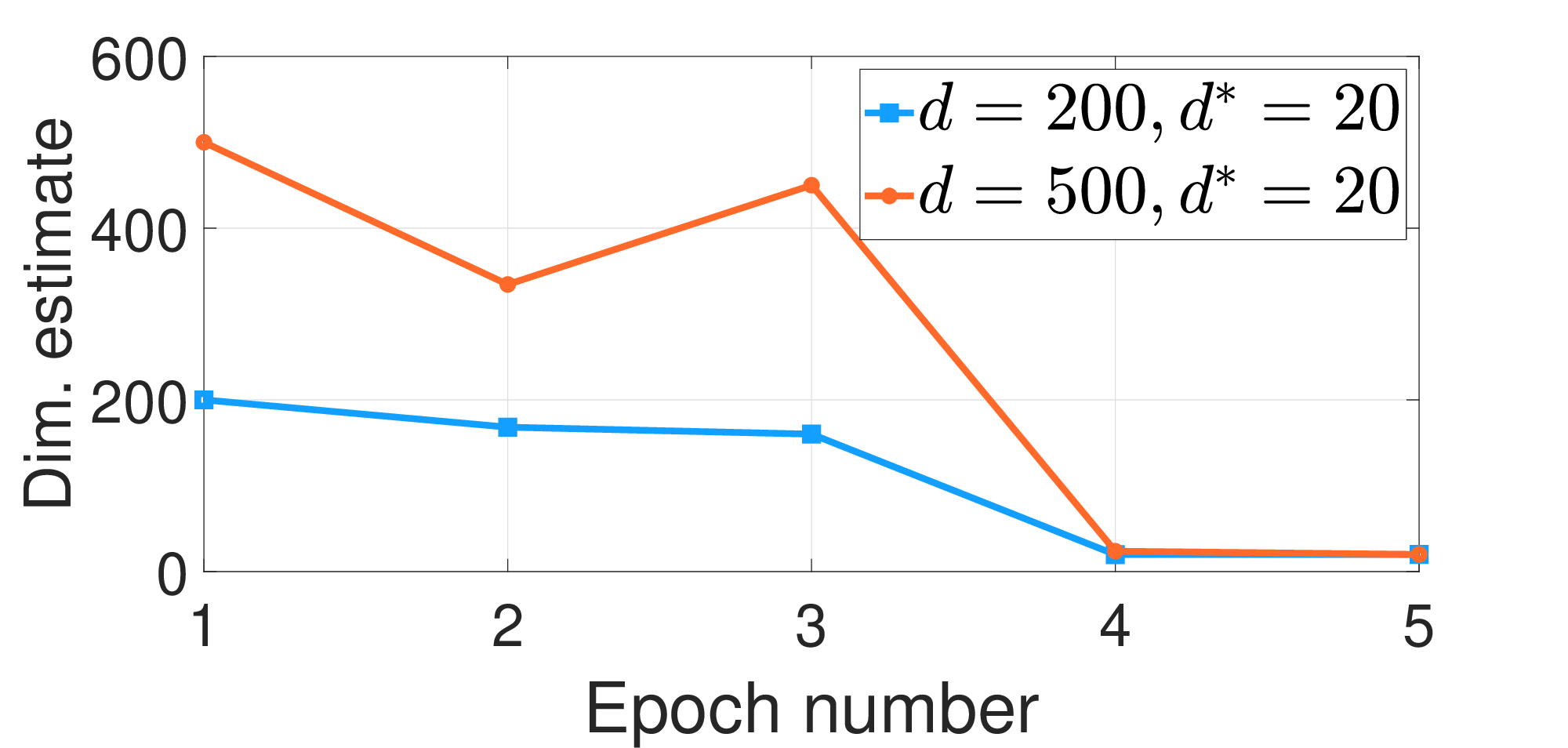} }}
    
    \subfigure[Comparison I ]{{\includegraphics[height = 4.2cm,width=4.2cm]{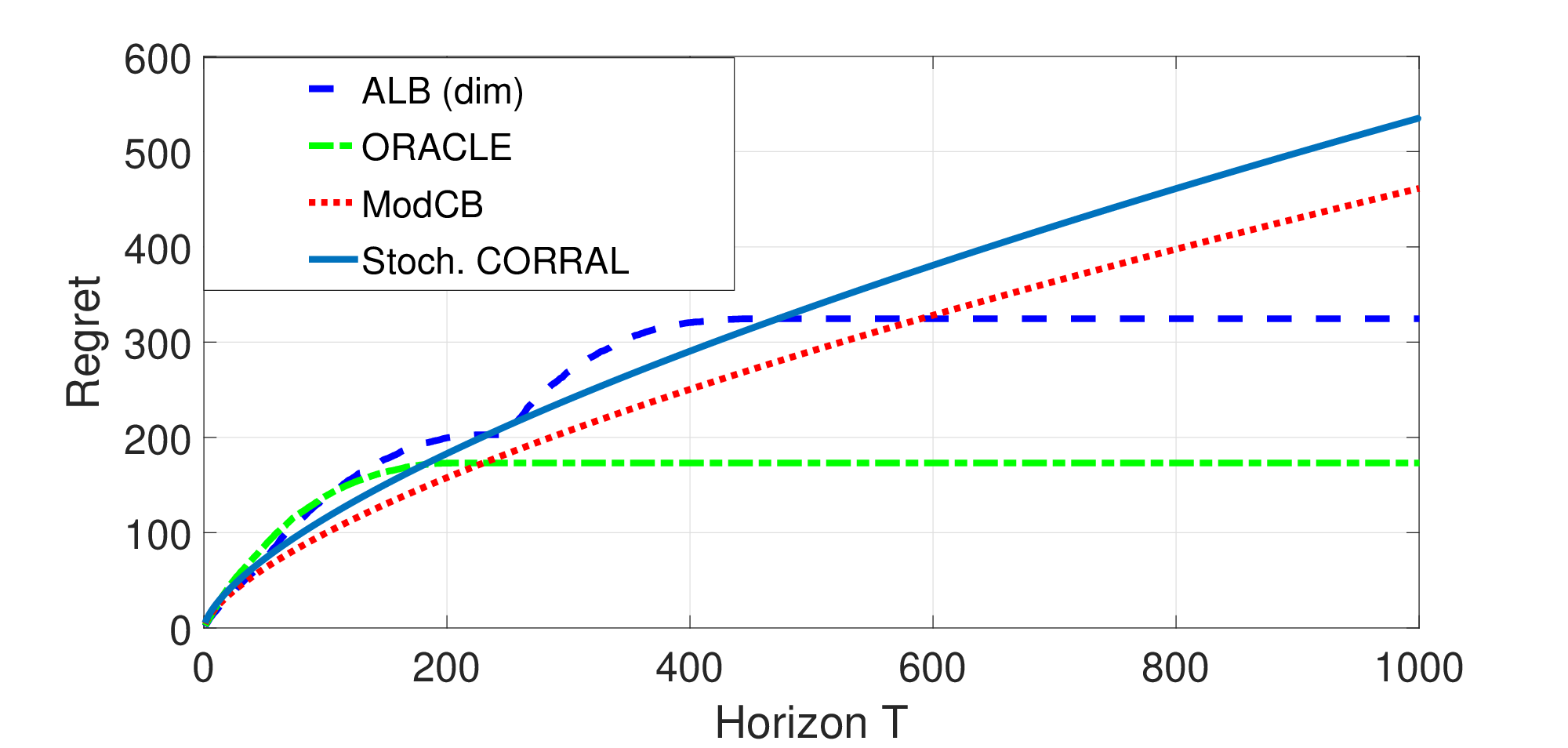} }}
    \subfigure[Comparison II ]{{\includegraphics[height = 4.2cm,width=4.2cm]{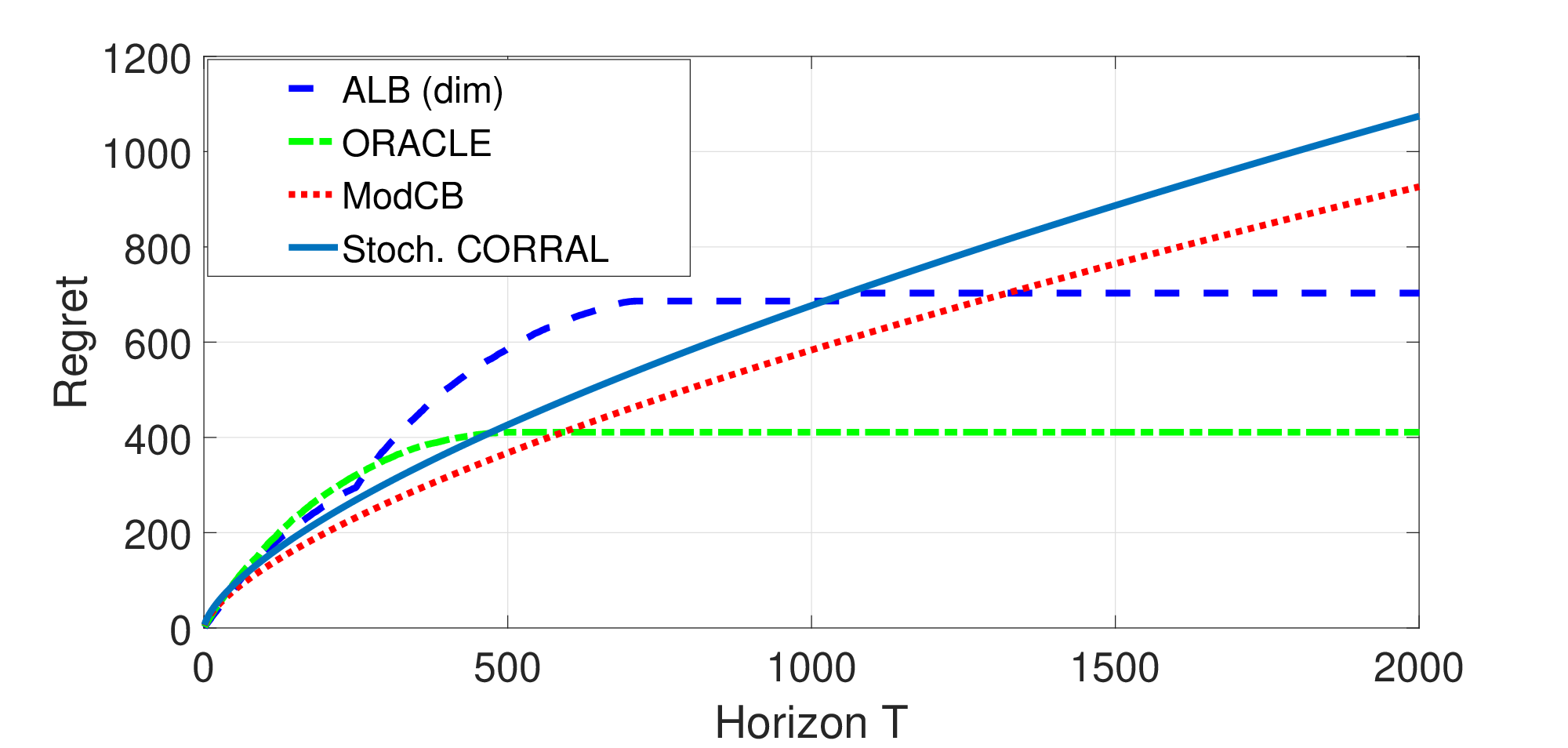} }}
    \caption{Synthetic experiments, validating the effectiveness of Algorithm~\ref{algo:main_algo_dimensions_unknown} and comparisons with several baselines. All the results are averaged over $25$ trials. }%
    \label{fig:experiments}
\end{figure*}
In this section we will verify the theoretical findings. We concentrate on the linear contextual bandit setup. We compare {\ttfamily ALB-Dim} with the (non-adaptive) OFUL algorithm of \cite{oful} and an \emph{oracle} that knows the problem complexity apriori. The oracle just runs OFUL with the known problem complexity. At each round of the learning algorithm, we sample the context vectors from a $d$-dimensional standard Gaussian, $\mathcal{N}(0,I_d)$. The additive noise to be zero-mean Gaussian random variable with variance $0.5$.

In panel (a)-(c), we compare the performance of {\ttfamily ALB-Dim} with OFUL (\cite{oful}) and an \emph{oracle} who knows the true support of $\theta^*$ apriori. 
For computational ease, we set $\varepsilon_i = 2^{-i}$ in simulations.
We select $\theta^*$ to be $d^* =20$-sparse, with the smallest non-zero component, $\gamma = 0.12$. We have $2$ settings: (i) $d = 500$ and (ii) $d = 200$. In panel (d) and (e), we observe a huge gap in cumulative regret between {\ttfamily ALB-Dim} and OFUL, thus showing the effectiveness of dimension adaptation. In panel (c), we plot the successive dimension refinement over epochs. We observe that within $4-5$ epochs, {\ttfamily ALB-Dim} finds the sparsity of $\theta^*$.

\paragraph*{Comparison of ALB (dim):} When $\theta^*$ is sparse, we compare ALB-Dim with $3$ baselines: (i) the ModCB algorithm of \cite{foster_model_selection} (ii) the Stochastic Corral algorithm of \cite{aldo-corral} and (iii) an oracle which knows the support of $\theta^*$. We select $\theta^*$ to be $d^* =20$ sparse, with dimension $d =200$ and $d=500$. The smallest non-zero component of $\theta^*$ is $0.12$. For ModCB, we use ILOVETOCONBANDITS algorithm, similar to \cite{agarwal2014taming}. We select the cardinality of action set as $2$ and select the sub-Gaussian parameter of the embedding as unity. 
In Figures 1(d) and 1(e), we observe that, the regret of ALB (dim) is better than ModCB and Stochastic Corral. The theoretical regret bound for ModCB scales as $\mathcal{O}(T^{2/3})$ (which is much larger than the ALB-Dim algorithm we propose), and Figure 1(c), validates this. The Stochastic Corral algorithm treats the base algorithms as bandit arms (with bandit feedback), as opposed to ALB-Dim which, at each arm-pull, updates the information about all the base algorithms. Thus, (Figs $1(d)$, $1(e)$), ALB-Dim has a superior performance compared to Stochastic Corral.

\ifCLASSOPTIONcaptionsoff
  \newpage
\fi

\bibliographystyle{IEEEtran}
\bibliography{ref-model-selection}

\begin{IEEEbiographynophoto}{Avishek Ghosh} (Ph.D UC Berkeley, 2021)
is an Assistant Professor at the department of Systems and Control Engg. and The Centre for Machine Intelligence and Data Science at IIT Bombay. Previously, he was an HDSI (Data Science) Post-doctoral fellow at the University of California, San Diego. Prior to this, he completed my PhD from the Electrical Engg. and Computer Sciences (EECS) department of UC Berkeley, advised by Prof. Kannan Ramchandran and Prof. Aditya Guntuboyina. His research interests are broadly in Theoretical Machine Learning, including Federated Learning and multi-agent Reinforcement/Bandit Learning. In particular, Avishek is interested in theoretically understanding challenges in multi-agent systems, and competition/collaboration across agents. Before coming to Berkeley, Avishek completed his masters degree from Indian Institute of Science (IISc), Bangalore (at the Electrical Communication Engg. Dept) and prior Avishek completed his  bachelors degree from Jadavpur University, in the dept. of Electronics and Telecommunication Engineering.
\end{IEEEbiographynophoto}

\begin{IEEEbiographynophoto}{Abishek Sankararaman}
(Ph.D, UT Austin 2019) is a Senior Applied Scientist at Amazon (AWS) where he conducts research on online learning and anomaly detection. Before AWS, he was a post-doctoral researcher at University of California, Berkeley, hosted by Prof. Venkat Anantharam, where he conducted research on networked learning in multi-armed bandits. Abishek received his PhD  from The University of Texas at Austin, where he was affiliated with the Simons Center for Network Mathematics and advised by Prof. François Baccelli. His PhD dissertation was based on analyzing novel stochastic geometric models for wireless dynamics and spatial random graph clustering and proving several phase-transition results on these models. Prior to this, he completed his undergraduate degree from IIT Madras. 
\end{IEEEbiographynophoto}

\begin{IEEEbiographynophoto}{Kannan Ramchandran}
(Ph.D Columbia University, 1993) is a Professor of Electrical Engineering and Computer Science at UC Berkeley, where he has been since 1999. He was on the faculty at UIUC from 1993 to 1999, and with AT$\&$T Bell Labs from 1984 to 1990.Prof. Ramchandran is a Fellow of the IEEE. He has published extensively in his field, holds over a dozen patents, and has received several awards for his research and teaching including an IEEE Information Theory Society and Communication Society Joint Best Paper award for 2012, an IEEE Communication Society Data Storage Best Paper award in 2010, two Best Paper awards from the IEEE Signal Processing Society in 1993 and 1999, an Okawa Foundation Prize for outstanding research at Berkeley in 2001, and an Outstanding Teaching Award at Berkeley in 2009, and a Hank Magnuski Scholar award at Illinois in 1998.
His research interests are broadly in the area of distributed systems theory and algorithms intersecting the fields of signal processing, communications, coding and information theory, and networking. His current systems focus is on large-scale distributed storage, large-scale collaborative video content delivery, and biological systems, with research challenges including latency, privacy and security, remote synchronization, sparse sampling, and shotgun genome sequencing.
\end{IEEEbiographynophoto}








\end{document}